\documentclass[preprint]{article}

\PassOptionsToPackage{numbers, compress}{natbib}
\usepackage{neurips_2025}

\usepackage[utf8]{inputenc} % allow utf-8 input
\usepackage[T1]{fontenc}    % use 8-bit T1 fonts
\usepackage{hyperref}       % hyperlinks
\usepackage{url}            % simple URL typesetting
\usepackage{booktabs}       % professional-quality tables
\usepackage{amsfonts}       % blackboard math symbols
\usepackage{nicefrac}       % compact symbols for 1/2, etc.
\usepackage{microtype}      % microtypography
\usepackage[dvipsnames]{xcolor}         % colors
\usepackage{amsmath,amsthm,amssymb}
\usepackage{dsfont}

\newtheorem{theorem}{Theorem}
\newtheorem{proposition}{Proposition}
\newtheorem{lemma}{Lemma}

\newtheorem{problem}{Problem}
\newtheorem{assumption}{Assumption}
\usepackage{tikz,amsmath,graphicx}
\usetikzlibrary{positioning, shapes.multipart, arrows.meta}
\usetikzlibrary{bending}
\usetikzlibrary{calc} 
\usepackage{geometry}
\usepackage{float}
\usepackage{subcaption}
\usepackage{enumitem}

\usepackage{algorithm}
\usepackage{algpseudocode}
\usepackage{bm}

\newcommand{\mycomment}[1]{}

\input{mysymbol.sty}

\title{A Few Moments Please: Scalable Graphon Learning via Moment Matching}

\newcommand{\vt}[1]{{\color{blue}{[Victor: #1]}}}
\newcommand{\rr}[1]{{\color{Mahogany}{[Reza: #1]}}}
\newcommand{\sms}[1]{{\color{cyan}{[Santiago: #1]}}}

\author{%
Reza Ramezanpour\\
  Rice University\\
  \texttt{rr68@rice.edu}\\
  \And
  Victor M. Tenorio\\
  King Juan Carlos University\\
  \texttt{victor.tenorio@urjc.es}\\
  \AND
  Antonio G. Marques \\
  King Juan Carlos University\\
  \texttt{antonio.garcia.marques@urjc.es}\\
  \And
  Ashutosh Sabharwal\\
  Rice University\\
  \texttt{ashu@rice.edu}\\
  \And
  Santiago Segarra \\
  Rice University\\
  \texttt{segarra@rice.edu} \\
}

\begin{document}

\maketitle

\begin{abstract}
Graphons, as limit objects of dense graph sequences, play a central role in the statistical analysis of network data. 
However, existing graphon estimation methods often struggle with scalability to large networks and resolution-independent approximation, due to their reliance on estimating latent variables or costly metrics such as the Gromov-Wasserstein distance. 
In this work, we propose a novel, scalable graphon estimator that directly recovers the graphon via moment matching, leveraging implicit neural representations (INRs). Our approach avoids latent variable modeling by training an INR--mapping coordinates to graphon values--to match empirical subgraph counts (i.e., moments) from observed graphs.
This direct estimation mechanism yields a polynomial-time solution and crucially sidesteps the combinatorial complexity of Gromov-Wasserstein optimization. 
Building on foundational results, we establish a theoretical guarantee: when the observed subgraph motifs sufficiently represent those of the true graphon (a condition met with sufficiently large or numerous graph samples), the estimated graphon achieves a provable upper bound in cut distance from the ground truth. Additionally, we introduce MomentMixup, a data augmentation technique that performs mixup in the moment space to enhance graphon-based learning. 
Our graphon estimation method achieves strong empirical performance--demonstrating high accuracy on small graphs and superior computational efficiency on large graphs--outperforming state-of-the-art scalable estimators in 75\% of benchmark settings and matching them in the remaining cases. Furthermore, MomentMixup demonstrated improved graph classification accuracy on the majority of our benchmarks.
\end{abstract}

\section{Introduction}
\label{sec:intro}

Networks are fundamental structures for representing complex relational data across diverse domains, from social interactions and biological systems to technological infrastructures~\citep{dong2020gspml,tenorio2025tracking}.
Understanding the underlying principles governing these networks is crucial for tasks such as link prediction~\citep{rey2023robust}, community detection~\citep{su2024community}, and, more broadly, node or graph classification~\citep{rey2025redesigning}.
Graphons, or graph limits, have emerged as a powerful mathematical framework for capturing the asymptotic structure of sequences of dense graphs \citep{lovasz2006limits,lovasz2012large,borgs2008convergent,avella2020centrality}.
They provide a continuous, generative model for graphs, enabling principled statistical analysis and offering a canonical representation for large networks.
Graphons have been successfully applied to derive controllers for large networks~\citep{gao2020graphon}, to understand network games with many actors~\citep{parise2023graphon}, to perform data augmentation in graph settings~\citep{navarro2023graphmad,han2022gmixup_icml}, and to aid in the topology inference of partially observed graphs~\citep{roddenberry2021network,navarro2022joint}.
As such, developing accurate and efficient methods for estimating graphons from observed network data is a central problem in network science and machine learning.

Estimating graphons from finite, potentially noisy graph observations presents significant challenges.
Many existing approaches suffer from computational scalability issues when applied to large networks~\citep{chatterjee2015matrix,xu2021learning}.
Furthermore, their resolution is limited by the size of the sample graphs, and obtaining a resolution-free approximation of the underlying continuous graphon can be difficult~\citep{chatterjee2015matrix}.
For instance, implicit neural representations (INRs) have been explored for graphon estimation due to their ability to model continuous functions~\citep{xia2023implicit}.
However, estimating the latent variables of the nodes to train the INRs remains a challenge, and oftentimes the literature resorts to computationally demanding optimal transport-inspired losses, like the Gromov-Wasserstein (GW) distance for optimization.
While recent scalable methods have made progress~\citep{azizpour2025scalable}, there remains a need for estimators that combine high accuracy, computational efficiency, and direct graphon recovery without complex intermediate steps.

In this paper, we introduce a novel and scalable approach for graphon estimation via moment matching, designed to overcome these prevalent limitations. 
Our method directly recovers the graphon by leveraging subgraph counts (graph moments) from observed data, thereby bypassing the need for latent variables and their associated complexities.
We represent the graphon using an INR, a continuous function parameterized by a neural network that maps coordinates in $[0,1]^2$ to the corresponding graphon value. 
The parameters of this INR are learned by minimizing the discrepancy between the moments derived from the INR and the empirical moments computed from the input graph(s). 
This direct recovery strategy, crucially, leads to a polynomial-time estimation algorithm and does not rely on combinatorial GW distances, distinguishing it from approaches like IGNR~\citep{xia2023implicit}. 
Our approach is underpinned by a theoretical result, building upon foundational work on convergent graph sequences~\citep{borgs2008convergent}, which establishes that if the motifs (subgraph patterns) in the observed graph data sufficiently represent the motifs present in the true underlying graphon--a condition met with sufficiently large or numerous graph samples--then the cut distance between the estimated and true graphons is provably upper bounded. Additionally, we propose MomentMixup, a novel data augmentation technique that operates by interpolating graph moments between classes and then learning the corresponding mixed graphons, offering an improvement over existing mixup strategies in the graphon domain~\citep{navarro2023graphmad,han2022gmixup_icml}.

Our contributions are threefold:
\begin{enumerate}
    \item We propose MomentNet, a scalable graphon estimator based on moment matching with INR, offering a resolution-free and estimation recovery mechanism.
    \item We provide a theoretical guarantee linking the fidelity of motif representation in observed data to the estimation accuracy in terms of cut distance.
    \item We introduce MomentMixup, an effective data augmentation method in the moment space for graphon-based learning tasks.
\end{enumerate}

The remainder of this paper is structured as follows: Section~\ref{sec:preliminaries} presents the necessary background concepts and related works. Section~\ref{sec:method} details our moment-matching INR approach for graphon estimation, including its theoretical characterization. Section~\ref{sec:moment_mixup} introduces MomentMixup, our approach for data augmentation in graph classification tasks. Section~\ref{sec:exps} presents our comprehensive empirical evaluations in both synthetic graphon estimation and data augmentation for graph classification. Finally, Section~\ref{sec:conclusion} concludes the paper and discusses future directions.

\section{Background, Related Works and Problem Formulation}
\label{sec:preliminaries}

In this section, we introduce the foundational concepts of graphons, motif densities, INRs for graphon estimation, and mixup for data augmentation. 
We also formally state the graphon estimation problem addressed in this paper.
In all these topics, we provide a summary of the literature, although a detailed discussion of related works can be found in Appendix A.

\paragraph{Graphons}
A graphon, short for ``graph function,'' is a fundamental concept in the theory of graph limits, serving as a limit object for sequences of dense graphs \citep{lovasz2012large, borgs2008convergent}. 
Formally, a graphon $W$ is a symmetric measurable function $W: [0,1]^2 \to [0,1]$.
Intuitively, the unit interval $[0,1]$ can be thought of as a latent space for the graph nodes. 
For any two points $x, y \in [0,1]$ (representing latent positions), the value $W(x,y)$ represents the probability of an edge forming between nodes associated with these latent positions.

A random graph $G_n(W)$ with $n$ nodes can be generated from a graphon $W$ by sampling $n$ i.i.d. latent positions $\eta_1, \eta_2, \ldots, \eta_n \sim \mathcal{U}[0,1]$ and, for each pair of distinct nodes $(i,j)$ with $1 \le i < j \le n$, an edge $(i,j)$ is included in $G_n(W)$ independently with probability $W(\eta_i, \eta_j)$.
Graphons are inherently invariant to permutations of node labels in the generated graphs, meaning that different orderings of the latent positions $\eta_i$ that preserve their relative positions in $[0,1]$ (or more formally, measure-preserving bijections of $[0,1]$) lead to equivalent graphon representations. 
The natural distance metric capturing this invariance is the cut distance~\citep{borgs2008convergent}.

\paragraph{Motif Densities from Graphons}\label{sec:density_funcs}
A key property of graphons is their ability to characterize the expected density of small subgraphs, often called motifs~\citep{borgs2008convergent,lovasz2012large}. For a simple graph $F$ (the motif), whose node and edge set are represented by $\ccalV_F$ and $\ccalE_F$, respectively, with $k = |\ccalV_F|$, its homomorphism density in a graphon $W$, denoted $t(F,W)$, is defined as
\begin{equation}
\label{eq:motif_density}
t(F,W) = \int_{[0,1]^k} \prod_{(i,j) \in \ccalE_F} W(\eta_i, \eta_j) \prod_{l \in \ccalV_F} d\eta_l.
\end{equation}
This integral represents the probability that $k$ randomly chosen latent positions from $[0,1]$ induce a subgraph homomorphic to $F$ according to the edge probabilities defined by $W$.
For a sufficiently large graph $G$ sampled from $W$, the empirical count of motif $F$ in $G$, normalized appropriately, converges to $t(F,W)$. Thus, empirical motif densities from observed graphs can serve as estimators for the true motif densities of the underlying graphon. The set of all such motif densities $\{t(F,W)\}_{F \in \ccalF}$ (for some collection of motifs $\ccalF$) is often referred to as the moment vector of the graphon~\citep{borgs10moments}. 
We also introduce the induced motif densities as follows 
\begin{align}\label{eq:induced_motif_density}
     t^{\prime}(F, W) = \int_{[0,1]^k} \prod_{(i,j) \in \mathcal{E}_F} W(\eta_i, \eta_j)  \prod_{(i,j) \notin \mathcal{E}_F} (1-W(\eta_i, \eta_j)) \prod_{l \in \mathcal{V}_F} d\eta_i.
\end{align}

This formulation for induced motif density, $t^{\prime}(F, W)$, specifically counts instances where the motif $F$ appears in $W$ with an \emph{exact} structural match. This means it accounts for both the required presence of edges specified in $F$ and the required \emph{absence} of edges between the motif's vertices that are not in $F$. In contrast, a non-induced (or homomorphism) density $t(F,W)$ only requires the presence of edges from $F$ in $W$, without any assumption of the value of the graphon associated with pairs of nodes not linked by an edge.

\paragraph{Implicit Neural Representations for Graphon Estimation}
\label{sec:inr_background}
An INR can effectively model a graphon by learning it as a continuous function~\citep{xia2023implicit,azizpour2025scalable}.
In this setup, the INR, typically a neural network $f_\theta: [0,1]^2 \to [0,1]$, is trained to take pairs of latent node coordinates $(\eta_i, \eta_j)$ from a continuous space as input, where $\eta_i$ and $\eta_j$ represent the latent positions associated with entities $i$ and $j$.
Its output is the predicted value of the graphon $f_\theta (\eta_i, \eta_j) = \hat{W} (\eta_i, \eta_j)$, representing the probability of an edge existing between these two latent positions.
The network $f_\theta$ learns this mapping from observed samples, which could be $( (\eta_{i_l}, \eta_{j_l}), W(\eta_{i_l}, \eta_{j_l}) )$ pairs derived from a large graph or a target graphon function, for a set of sample indices $l$.
Crucially, because $f_\theta$ learns a continuous function over the entire input coordinate space defined by $\eta_\cdot$, the resulting graphon representation is inherently resolution-free.
This means it can determine the edge probability for any arbitrary pair of latent coordinates $(\eta_i, \eta_j)$, allowing for the generation or analysis of graph structures at any desired level of detail or scale without being tied to a fixed number of nodes or a specific discretization.
%The INR generates values for a subsequent moment estimation step by mapping inputs that have been sampled from a latent variable space.

\paragraph{Mixup for Data Augmentation}
\label{sub:mixup_augmentation}

The core idea of Mixup \citep{zhang2017mixup} is to generate synthetic training examples by taking convex combinations of pairs of existing samples and their corresponding labels. Given two input samples $x_i$ and $x_j$ with their respective labels $y_i$ and $y_j$, a new synthetic sample $(\tilde{x}, \tilde{y})$ is created as $\tilde{x} = \lambda x_i + (1-\lambda) x_j$, $\tilde{y} = \lambda y_i + (1-\lambda) y_j$.
where $\lambda \in [0,1]$ is a mixing coefficient. % typically sampled from a Beta distribution (e.g., $\text{Beta}(\alpha, \alpha)$ for some hyperparameter $\alpha > 0$). \sms{is this comment about the beta distribution relevant?}
This encourages the model to behave linearly in-between training examples, leading to smoother decision boundaries and improved generalization.

Applying Mixup directly to graph-structured data presents challenges because graphs are not inherently Euclidean objects. To perform Mixup for graphs, one typically first maps the graphs into a suitable Euclidean representation~\citep{navarro2023graphmad,han2022gmixup_icml}.
%This could be the node feature matrix, the adjacency matrix (or a flattened version), a graph embedding learned by a GNN or other graph embedding technique or a vector of structural properties, such as motif counts (moments).
For example, GraphMAD~\citep{navarro2023graphmad} maps the graphs to a latent space and performs nonlinear mixup, while G-Mixup~\citep{han2022gmixup_icml} performs mixup in the graphon domain.
Once graphs $G_i$ and $G_j$ are available as Euclidean representations $\bbz_i$ and $\bbz_j$ respectively, a mixed representation $\tbz = \lambda \bbz_i + (1-\lambda) \bbz_j$ can be computed. The subsequent step, which can be non-trivial, is to generate a new graph $\tilde{G}$ from this mixed representation $\tbz$ that can be used for training a graph classification model.

\paragraph{Problem Formulation.}
The primary problem addressed in the graphon estimation literature, and in this work, is to recover the underlying graphon $W^*$ given one or more observed graphs.

\begin{problem}[Graphon Estimation]
\label{prob:graphon_estimation}
Given a set of observed graphs $\mathcal{G} = \{G_1, G_2, \ldots, G_P\}$, where each $G_p$ has $n_p$ vertices and is assumed to be sampled (conditionally independently) from an unknown true graphon $W^*$, i.e., $G_p \sim G_{n_p}(W^*)$, the goal is to estimate $W^*$.
\end{problem}
In the literature, early methods aimed at solving Problem~\ref{prob:graphon_estimation} by means of histogram estimators and stochastic block models~\citep{borgs2008convergent,lovasz2012large,xing2014consistent,airoldi2008mixed,gao2015rate}. Other non-parametric approaches, like Universal Singular Value Thresholding (USVT)~\citep{chatterjee2015matrix}, aimed to recover underlying network structures but often faced computational or resolution limitations. More recent scalable techniques include those using INRs. For instance, IGNR~\citep{xia2023implicit} often leverages GW distances~\citep{peyre2016gromov,xu2019gromov,xu2021learning} for alignment, while methods like SIGL~\citep{azizpour2025scalable} further advance INR-based estimation.

Our work proposes a novel method for solving Problem~\ref{prob:graphon_estimation} by directly learning an INR to match empirical moments (subgraph counts) from the observed graph(s), thereby bypassing latent variables and computationally expensive metric optimizations.
Moreover, we leverage our proposed solution to Problem~\ref{prob:graphon_estimation} to design MomentMixup, a novel mixup strategy for graph data augmentation.
MomentMixup performs mixup in the space of empirical moments, offering a novel way to generate augmented graph data informed by the underlying generative structure.

\section{Moment Matching Neural Network (MomentNet)}\label{sec:method}

In the following subsections, we introduce our proposed method, \textbf{MomentNet}, for learning the graphon. We also provide the fundamental theorem upon which our model is built.

\subsection{Methodology}\label{sub:method}

\begin{figure}[t]
\centering

\begin{tikzpicture}[
  node distance=0.8cm and 0.8cm,
  every node/.style={font=\small},
  box/.style={draw, rounded corners, minimum width=4.3cm, minimum height=5.2cm},
  inner box/.style={draw, rounded corners, fill=#1!10, minimum width=3.6cm, minimum height=6.5cm},
  motif/.style={circle, draw, fill=white, minimum size=3mm, inner sep=0pt}, scale=0.9
  ]

%% === Box A ===
\node[box, minimum width=3.9cm] (a) at (-0.5,0) {};
\node at (-0.5,3.1) {(a) Observed graphs};

% Graphon image on the left
\node (graphonImg) at (-1.75,0.3) {\includegraphics[width=1.5cm]{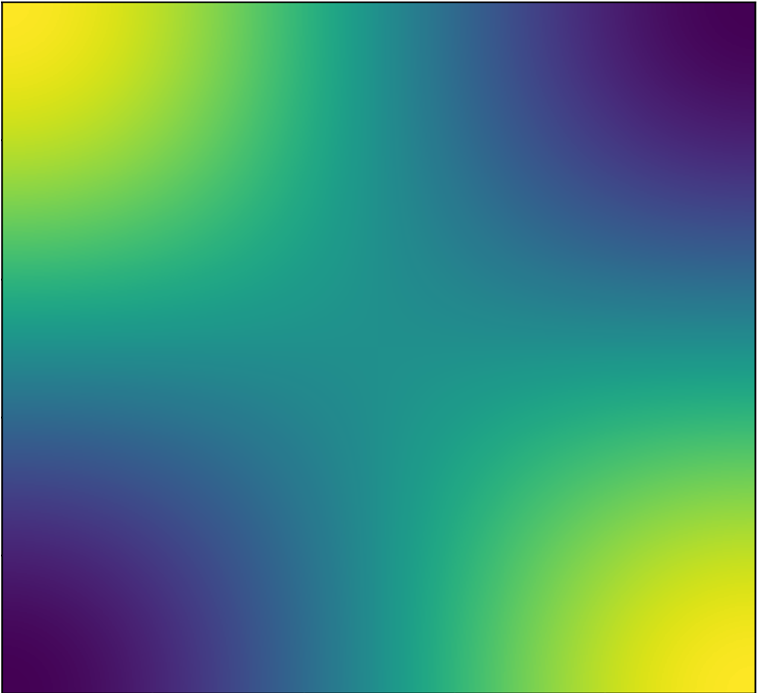}};
\node at (-1.75,-0.8) {\scriptsize \textbf{True graphon}};

% Sample graph 1 (top right)
\node[draw, inner sep=0.1pt] (sample1) at (0.45,1.6) {\includegraphics[width=1.8cm]{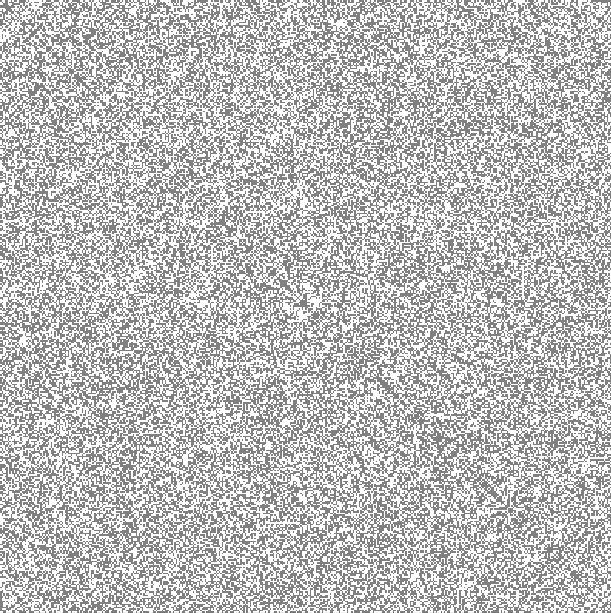}};

% Ellipsis in the middle
\node at (0.45,0.35) {$\vdots$};

% Sample graph 2 (bottom right)
\node[draw, inner sep=0.1pt] (sample2) at (0.45,-1.1) {\includegraphics[width=1.8cm]{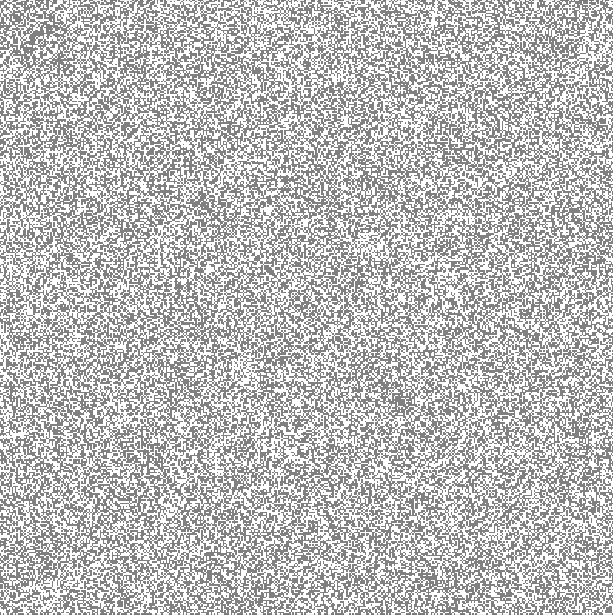}};
\node at (0.45,-2.4) {\scriptsize \textbf{Sampled graphs}};

\draw[->, thick, black!80] 
  ([xshift=-4pt]graphonImg.east) 
  to [out=10, in=175] 
  (sample1.west);

\draw[->, thick, black!80] 
  ([xshift=-4pt]graphonImg.east) 
  to [out=10, in=175] 
  (sample2.west);

%% === Box B ===

\node[box] (b) at (4.2,0) {};
\node at (4.2,3.1) {(b) Density of Motifs};

\tikzset{
  circ/.style ={circle, draw, fill=#1,   inner sep=0pt, minimum size=1.4mm},
  tri/.style  ={regular polygon, regular polygon sides=3,
                draw, fill=#1,   inner sep=0pt, minimum size=1.7mm}
}

% --------------------------------------------------------------
%   SAMPLE 1  (red)  ─────────────────────────────────────────
% --------------------------------------------------------------

\begin{scope}[shift={(1.47,0)}]             % local origin:  (x≈2.45, y≈0)
\node at (1.3,2.1) {\scriptsize $\textbf{G}_1$};

  % adjacency image
  \node[inner sep=0pt] (img1) at (1.2,1.2)
        {\includegraphics[width=1.3cm]{figs/graph1_sample.png}};

  % overlay three motifs: link, chain, triangle
  \begin{scope}[shift={(img1.south west)}, xshift=0.12cm, yshift=0.28cm, scale=0.23]
    \node[circ=red!80] (a) at (0,0) {};
    \node[circ=red!80] (b) at (1,0) {};
    \draw[thick] (a)--(b);
  \end{scope}
  \begin{scope}[shift={(img1.south west)}, xshift=0.48cm, yshift=0.60cm, scale=0.23]
    \node[circ=red!80] (a) at (0,2) {};
    \node[circ=red!80] (b) at (0.7,2.6) {};
    \node[circ=red!80] (c) at (1.4,2) {};
    \draw[thick] (a)--(b)--(c);
  \end{scope}
  \begin{scope}[shift={(img1.south west)}, xshift=0.80cm, yshift=0.28cm, scale=0.23]
    \node[circ=red!80] (a) at (1.3,0) {};
    \node[circ=red!80] (b) at (2.2,0) {};
    \node[circ=red!80] (c) at (2.3,1) {};
    \draw[thick] (a)--(b)--(c)--(a);
  \end{scope}

  % moment vector  m^{(1)}  with tiny motif pictograms
  \node (vec1) at (1.2,-0.35) {\scalebox{0.42}{$
    \mathbf{m}^{(1)}=\!\Bigl[
    \begin{tikzpicture}[baseline={([yshift=-.5ex]current bounding box.center)},scale=0.4]
      \node[circ=red!80] (a) at (0,0) {};
      \node[circ=red!80] (b) at (0.55,0) {};
      \draw[thick] (a)--(b);
    \end{tikzpicture},
    \begin{tikzpicture}[baseline={([yshift=-.5ex]current bounding box.center)},scale=0.4]
      \node[circ=red!80] (a) at (0,0) {};
      \node[circ=red!80] (b) at (0.50,0.40) {};
      \node[circ=red!80] (c) at (1.0,0) {};
      \draw[thick] (a)--(b)--(c);
    \end{tikzpicture},
    \begin{tikzpicture}[baseline={([yshift=-.5ex]current bounding box.center)},scale=0.4]
      \node[circ=red!80] (a) at (0,0) {};
      \node[circ=red!80] (b) at (0.55,0) {};
      \node[circ=red!80] (c) at (0.28,0.48) {};
      \draw[thick] (a)--(b)--(c)--(a);
    \end{tikzpicture},\dots\Bigr]$}};

  % arrow from image → vector
  \draw[->,thick,red!80] (img1.south) .. controls +(0,-0.22) and +(0,0.22) .. (vec1.north);
\end{scope}

\node[font=\Large\bfseries] at (4.2,1.2) {$\boldsymbol{\hdots}$};

% --------------------------------------------------------------
%  SAMPLE 2  (blue)  ────────────────────────────────────────
% --------------------------------------------------------------
\begin{scope}[shift={(4.2,0)}]  % new local origin
\node at (1.6,2.1) {\scriptsize $\textbf{G}_P$};

  % adjacency image
  \node[inner sep=0pt] (img2) at (1.5,1.2)
        {\includegraphics[width=1.3cm]{figs/graph2_sample.png}};

  % overlay three motifs
  \begin{scope}[shift={(img2.south west)}, xshift=0.12cm, yshift=0.28cm, scale=0.23]
    \node[circ=blue!80] (a) at (4,0) {};
    \node[circ=blue!80] (b) at (5,1) {};
    \draw[thick] (a)--(b);
  \end{scope}
  \begin{scope}[shift={(img2.south west)}, xshift=0.48cm, yshift=0.60cm, scale=0.23]
    \node[circ=blue!80] (a) at (-1,-1) {};
    \node[circ=blue!80] (b) at (-0.3,-0.4) {};
    \node[circ=blue!80] (c) at (0.4, -1) {};
    \draw[thick] (a)--(b)--(c);
  \end{scope}
  \begin{scope}[shift={(img2.south west)}, xshift=0.80cm, yshift=0.28cm, scale=0.23]
    \node[circ=blue!80] (a) at (-1,4) {};
    \node[circ=blue!80] (b) at (-0.2,3) {};
    \node[circ=blue!80] (c) at (1,3.7) {};
    \draw[thick] (a)--(b)--(c)--(a);
  \end{scope}

  % moment vector  m^{(2)}
  \node (vec2) at (1.5,-0.35) {\scalebox{0.42}{$
    \mathbf{m}^{(P)}=\!\Bigl[
    \begin{tikzpicture}[baseline={([yshift=-.5ex]current bounding box.center)},scale=0.4]
      \node[circ=blue!80] (a) at (0,0) {};
      \node[circ=blue!80] (b) at (0.55,0) {};
      \draw[thick] (a)--(b);
    \end{tikzpicture},
    \begin{tikzpicture}[baseline={([yshift=-.5ex]current bounding box.center)},scale=0.4]
      \node[circ=blue!80] (a) at (0,0) {};
      \node[circ=blue!80] (b) at (0.50,0.40) {};
      \node[circ=blue!80] (c) at (1.0,0) {};
      \draw[thick] (a)--(b)--(c);
    \end{tikzpicture},
    \begin{tikzpicture}[baseline={([yshift=-.5ex]current bounding box.center)},scale=0.4]
      \node[circ=blue!80] (a) at (0,0) {};
      \node[circ=blue!80] (b) at (0.55,0) {};
      \node[circ=blue!80] (c) at (0.28,0.48) {};
      \draw[thick] (a)--(b)--(c)--(a);
    \end{tikzpicture},\dots\Bigr]$}};

  % arrow  image → vector
  \draw[->,thick,blue!80] (img2.south) .. controls +(0,-0.22) and +(0,0.22) .. (vec2.north);
\end{scope}

% --------------------------------------------------------------
%  MEAN MOMENT VECTOR  ------------
% --------------------------------------------------------------
\node at (4.2,-1.7) {\scalebox{0.72}{$
  \mathbf{m}
  =\tfrac{1}{P}\sum_{p=1}^{P}\mathbf{m}^{(p)}
  =\!\Bigl[
  \begin{tikzpicture}[baseline={([yshift=-.5ex]current bounding box.center)},scale=0.45]
    \node[circ=black!65] (a) at (0,0) {};
    \node[circ=black!65] (b) at (0.60,0) {};
    \draw[thick] (a)--(b);
  \end{tikzpicture},
  \begin{tikzpicture}[baseline={([yshift=-.5ex]current bounding box.center)},scale=0.45]
    \node[circ=black!65] (a) at (0,0) {};
    \node[circ=black!65] (b) at (0.50,0.40) {};
    \node[circ=black!65] (c) at (1.0,0) {};
    \draw[thick] (a)--(b)--(c);
  \end{tikzpicture},
  \begin{tikzpicture}[baseline={([yshift=-.5ex]current bounding box.center)},scale=0.45]
    \node[circ=black!65] (a) at (0,0) {};
    \node[circ=black!65] (b) at (0.55,0) {};
    \node[circ=black!65] (c) at (0.28,0.48) {};
    \draw[thick] (a)--(b)--(c)--(a);
  \end{tikzpicture},\dots\Bigr]$}};

% arrows  m^{(1)}, m^{(2)} →  mean
\draw[->,thick] (3.2,-0.75) .. controls +(0,-0.22) and +(0,0.25) .. (4.2,-1.5);
\draw[->,thick] (5.2,-0.75) .. controls +(0,-0.22) and +(0,0.25) .. (4.2,-1.5);

%% === Box C ===
\node at (10.05,3.1) {(c) Moment network training};

\node[box, minimum width=6cm] (c) at (10.05,0) {};

\begin{scope}[shift={(7.15,1.6)}, scale=0.05]
\node (input) [xshift=0.4cm] {
  \scalebox{0.85}{$
    \left\{ \begin{array}{l}
    \eta^{(l)}_i \\
    \eta^{(l)}_j
    \end{array} \right\}_{l=1}^L$}
};

  \node (box) [draw=none, fill=blue!20, minimum width=0.8cm, minimum height=1.8cm, right of=input, xshift=0.8cm, rounded corners=4pt] 
        {$f_{\theta}$};

  \node (output) [right of=box, xshift=1.3cm] {\scalebox{0.9}{$\mathbf{s} = 
    \left\{f_\theta(\eta^{(l)}_i, \eta^{(l)}_j) \right\}_{l=1}^L$}};

  % Arrows
  \draw[->] (input) -- (box);
  \draw[->] (box) -- (output);

  % Label
  \node at ($(box.south) + (0, -4.8)$) {\textbf{INR}};
\end{scope}

\begin{scope}[shift={(7.3,-1)}, scale=0.05]
\node (input) [xshift=0cm] {$\mathbf{s}$};

  \node (box) [draw=none, fill=green!20, minimum width=1.5cm, minimum height=1cm, right of=input, xshift=0.7cm, rounded corners=4pt] 
        {\scriptsize\textbf{Moment estimator}};

  \node (wmse) [draw, fill=red!20, minimum height=1.7cm, right of=box, xshift=1.3cm, yshift=-0.4cm, rounded corners=4pt] {\scriptsize $\min_{\theta}$ \textbf{WMSE}};

\node (secondin) [left of=wmse, xshift=-0.6cm, yshift=-0.5cm] {\scriptsize $\mathbf{m}$};

\node (outputtheta) [right of=wmse, xshift=0.5cm, yshift=0cm] {\scriptsize $\theta_{\star}$};

% Arrows
\draw[->] (input) -- (box);
\draw[->] (box) -- ($(wmse.west)+(0,8.5)$);

\draw[->] (secondin) -- ($(wmse.west)+(0,-11)$);
\draw[->] (wmse) -- (outputtheta);
  
\end{scope}
\end{tikzpicture}
\caption{Graphon estimation pipeline: observed graphs lead to motif frequency extraction and INR-based recovery.}
\label{fig:MomentNet}
\end{figure}

We explain the two steps in our method to estimate the graphon $W$ given the set of sampled graphs denoted by $\ccalG = \left\{G_p\right\}_{p=1}^{P}$. 
A schematic view of our method is presented in Figure~\ref{fig:MomentNet}.

\paragraph{Step 1: Computing density of motifs.}
For each graph in our dataset $\mathcal{G}$, we count the occurrences of specific motifs. The density of an identified motif is then calculated as the ratio of its observed count to the total number of possible ways that particular motif could appear in a graph of the same size. We use the ORCA method~\cite{orca} to count the number of graphlets in the graph, and then we convert the graphlet count into motif counts.
This aggregation is needed because our analysis cares only about how often each subgraph pattern appears in total, not about the exact placement of individual nodes within those patterns; see Fig.~1 in~\cite{orca} for an illustration. 
We parallelize the use of the ORCA method for computing motif counts across the graphs in our dataset, thereby gaining a significant speed-up in processing time.
ORCA can count motifs with up to five nodes, and its method can be extended to handle larger motifs. Once these motif-based statistics are calculated from the graphs, we no longer use the graphs themselves for subsequent steps. 
This approach significantly reduces computational overhead. 
Mathematically, we consider a set $\mathcal{F}$ of $|\mathcal{F}|$ distinct motif types. For each graph $G_p$ in our dataset, its empirical motif density vector is $\mathbf{m}^{(p)} \in \mathbb{R}^{|\mathcal{F}|}$. The overall motif density vector $\mathbf{m}$ for the dataset is currently computed as the simple average:
\begin{align} \label{eq:simple_avg_motif_density}
    \mathbf{m} = \frac{1}{P} \sum_{p=1}^{P} \mathbf{m}^{(p)}.
\end{align}
While Eq.~\eqref{eq:simple_avg_motif_density} treats each graph equally, a more general approach could involve a weighted average, $\mathbf{m}_w = \sum_{p=1}^{P} w_p \mathbf{m}^{(p)}$ (where $w_p \ge 0$ and $\sum_{p=1}^{P} w_p = 1$). Such weights $w_p$ could, for example, depend on graph properties like size ($n_p$), potentially giving more influence to larger graphs, which might yield more stable density estimates. Our present work employs the simple average, with the exploration of weighted schemes as a potential future refinement.

\paragraph{Step 2: Training the Moment network}
The moment network is defined as a combination of INR with a moment-based loss function. This step consists of three components described as follows:
\begin{enumerate}
    \item \textbf{INR}: Our methodology employs an INR $f_\theta$ to model the graphon, that receives the latent coordinates $(\eta_i, \eta_j)$ and outputs the estimated graphon value $\hat{W}_\theta (\eta_i, \eta_j)$, as explained in Section~\ref{sec:inr_background}.
%In our specific application, the trained INR $f_\theta$ aims to estimate an underlying true graphon. It generates values $f_{\theta}(\eta_i, \eta_j)$, which represent these estimates of the true graphon at the specific latent coordinates $(\eta_i, \eta_j)$. These estimates are produced by providing $f_\theta$ with pairs of latent coordinates $(\eta_i, \eta_j)$ that are sampled from the latent variable space. The generated estimates $f_{\theta}(\eta_i, \eta_j)$ then serve as direct inputs for the subsequent moment estimation step. This approach leverages the INR's inherent resolution-free property for flexible querying of the learned graphon estimate.
    
    \item \textbf{Moment estimator}:
    With the graphon estimated by the INR function $f_{\theta}$, we can compute the induced motif density for any given motif $F$. This is achieved by substituting $\hat{W}_{\theta}$ in place of $W$ in~\eqref{eq:induced_motif_density}.
    % the formula for $t'(F,W)$ that was introduced in Section~\ref{sec:density_funcs}.
    Since we can not compute the integral directly, we approximate it using Monte Carlo integration techniques. By generating a sufficient number of random samples $L$ from the distribution induced by the graphon, we can estimate the integral. More precisely, we sample $L$ samples of $k$ latent coordinates $\eta_1^{(l)}, \ldots, \eta_k^{(l)}$, where $\eta_i^{(l)} \sim \ccalU [0,1]$. Then we estimate \( t^{\prime}(F, \hat{W}_{\theta}) \) as
    \begin{align}\label{eq:mc_moment}
        \hat{t}^{\prime} (F, \hat{W}_{\theta}) = \frac{1}{L} \sum_{l=1}^{L} \left[ \prod_{(i,j) \in \mathcal{E}_F} \hat{W}_{\theta}(\eta^{(l)}_{i}, \eta^{(l)}_{j})  \prod_{(i,j) \notin \mathcal{E}_F} (1-\hat{W}_{\theta}(\eta^{(l)}_{i}, \eta^{(l)}_{j})) \right]
    \end{align}
    %$\eta^{(t)}_{i}$ denotes to latent variable assigned to node $i$ in sample $t$.

    The Monte Carlo estimator $\hat{t}^{\prime}(F, \hat{W}_{\theta})$ is differentiable with respect to the INR parameters $\theta$. Since the INR $f_{\theta}$ is a neural network parameterized by $\theta$ (and thus differentiable with respect to $\theta$), the estimator $\hat{t}^{\prime}(F, \hat{W}_{\theta})$, which is constructed as an average of terms derived from $f_{\theta}$ outputs at fixed sample points $\boldsymbol{\eta}^{(l)}$, is consequently also differentiable with respect to $\theta$.
    This characteristic is vital as it allows the use of gradient-based optimization algorithms to train the INR parameters $\theta$ when this estimator is incorporated into a loss function. 
    A proof of unbiasedness for this approach, i.e., showing that $\mathbb{E}[\hat{t}^{\prime}(F, \hat{W}_{\theta})] = t^{\prime}(F, \hat{W}_{\theta})$, is provided in Appendix~\ref{app:unbiased}.
    The vector of estimated moments (e.g., motif densities) derived from the INR outputs is denoted as $\hat{\mathbf{m}}(\theta) = \left[ \hat{t}^{\prime}(F_1, \hat{W}_{\theta}), \hat{t}^{\prime}(F_2, \hat{W}_{\theta}), \ldots, \hat{t}^{\prime}(F_{|\mathcal{F}|}, \hat{W}_{\theta}) \right]^\top$. 
    \item \textbf{WMSE}:
    We use weighted mean squared error as a loss function to train our INR. Given the empirical moment vector $\mathbf{m}$, based on sampled graphs and computed using Eq.~\eqref{eq:simple_avg_motif_density}, and the estimated moments based on the INR as $\hat{\mathbf{m}}(\theta)$, the loss function is
        \begin{align} 
        L(\theta) = \sum_{i=1}^{|\mathcal{F}|} w_i \left( m_i - \hat{m}_i(\theta) \right)^2.
    \end{align}
   In our experiments, we adjust the importance of different factors by assigning weights ($w_i$). We calculate these weights as the inverse of how strong each factor ($m_i$) appears in our data ($w_i = \frac{1}{m_i}$). This weighting method balances the impact of each moment, preventing the most frequent ones (larger $m_i$) from having a large effect on the learning process.

\end{enumerate}

The training process described above, optimizing the parameters $\theta$ of the INR $f_\theta$ to minimize the weighted mean squared error between empirical and estimated motif densities, yields our final graphon estimate $\hat{W}_\theta = f_\theta$. This estimated graphon is inherently scale-free due to the continuous nature of the INR. Furthermore, the entire estimation procedure operates in polynomial time with respect to the number of nodes and motifs considered.
A detailed complexity analysis is provided in Appendix~\ref{app:complexity}.

\subsection{Theoretical characterization}
\label{sub:theory}

We present our main theorem bounding the cut distance between the true graphon $W^*$ and the graphon $\hat{W}_\theta$ estimated by our proposed INR. 
This result combines insights from the concentration of empirical motif densities in the $G_n (W)$ model~\cite{borgs2008convergent} with the inverse counting lemma relating motif distances to cut distance, and an assumption about the neural network's ability to approximate empirical motif densities. 
Supporting lemmas and the proof of this theorem are provided in Appendix~\ref{app:proof_theo}.

Let $G_1, \dots, G_p$ be $P$ graphs, each with $n$ vertices, sampled independently from the graphon model $G_n (W^*)$ according to the graphon $W^*$. 
The \textbf{empirical motif density} of $F$ based on these samples is $\bar{t} (F,W^*) = \frac{1}{P} \sum_{p=1}^P t(F,G_p)$, where in a slight abuse of notation we denote by $t(F,G_p)$ the motifs densities computed from the motif counts of graph $G_p$.

\mycomment{
\sms{ so the result is for $t$ but in the algorithm we use $t'$?} \vt{Yes, I agree that it is a bit weird, but in the experiments we are using $t'$, but all the results this is based on talk about $t$.} 
\sms{I do not understand the abuse of notation here. Which one is graph $G$? we do not have a graph $G$, right?} \vt{$t(F, W)$ receives as arguments a simple graph and a graphon, and it is defined as such in~\eqref{eq:motif_density}. It is not defined to receive a graph in its second argument (equation~\eqref{eq:motif_density} cannot be applied). $G$ was a generic graph, I have changed it to be $G_p$.}}

We consider an INR $f_\theta$ with parameters $\theta$, whose estimated graphon is denoted by $\hat{W}_\theta$. The motif densities corresponding to this estimated graphon are denoted by $\hat{t}_\theta(F, \hat{W}_\theta)$. The INR is trained to directly output $\hat{t}_\theta(F, \hat{W}_\theta)$ values that approximate the empirical densities $\bar{t} (F,W^*)$.
Finally, let $\mathcal{F}_k$ denote the set of all non-isomorphic simple graphs with exactly $k$ vertices and let $N_k = |\mathcal{F}_k|$ be the number of such graphs.

With the previous definitions, and those of Lemma~\ref{lem:motif_to_cut_distance} and Assumption~\ref{as:nn_perf} in Appendix~\ref{app:proof_theo}, we are in a position to present our main result, stated in Theorem~\ref{thm:main_cut_distance_bound}.

\mycomment{
\sms{A lot of redundancy between the text before the Theorem and the statement of the Theorem. No need to repeat everything. Also, you talk about assumptions that do not exist (presumably is in the appendix) ... do we need the assumption? if you are actually stating what that is. Also, instead of mentioning Lemma 2 first, present the equation (8) sooner and then say that these quantities appear in Lemma 2, but AFTER the equation. Otherwise, I need to read too much before I understand what the Theorem is about.}
\vt{I've removed some of the definitions previously stated. Also, I've left the assumption because I have a comment about it in the appendix that I think won't fit here, but I talk about the assumption in the previous paragraph now. I've also changed the order of the definitions and the result in the theorem, and left two possibilities, please choose whichever you think is better.}
}

\begin{theorem}[Cut Distance Bound for INR Estimated Graphons]
\label{thm:main_cut_distance_bound}
Assume $n > \frac{k(k-1)}{\delta_M}$, Assumption~\ref{as:nn_perf} holds, and
\begin{equation} \label{eq:sampling_cond}
N_k \cdot 2 \exp\left(-\frac{Pn}{4k^2}\left(\frac{\delta_M}{2} - \frac{k(k-1)}{2n}\right)^2\right) < \zeta,
\end{equation}
where $\zeta > 0$ is a desired confidence level and $\delta_M > 0$ is the motif deviation threshold defined in Lemma~\ref{lem:motif_to_cut_distance}.
Then, with probability at least $1-\zeta$, the cut distance between the neural network estimated graphon $\hat{W}_\theta$ and the true graphon $W^*$ is bounded by $\eta$ as
\begin{equation}
d_{\text{cut}}(\hat{W}_\theta, W^*) < \eta,
\end{equation}
where $\eta$ is also defined in Lemma~\ref{lem:motif_to_cut_distance}.
\end{theorem}

\mycomment{
\begin{theorem}[Cut Distance Bound for INR Estimated Graphons]
\label{thm:main_cut_distance_bound}
Let $\delta_M > 0$ and $\eta$ be the motif deviation threshold and cut distance upper bound defined in Lemma~\ref{lem:motif_to_cut_distance}, respectively.
Also, let $1-\zeta > 0$ (where $0 < \zeta < 1$) be a desired confidence level, and assume $n > \frac{k(k-1)}{\delta_M}$, Assumption~\ref{as:nn_perf} holds, and
\begin{equation} \label{eq:sampling_cond}
N_k \cdot 2 \exp\left(-\frac{Pn}{4k^2}\left(\frac{\delta_M}{2} - \frac{k(k-1)}{2n}\right)^2\right) < \zeta.
\end{equation}
Then, with probability at least $1-\zeta$, the cut distance between the neural network estimated graphon $\hat{W}_\theta$ and the true graphon $W^*$ is bounded by $\eta$
\begin{equation}
d_{\text{cut}}(\hat{W}_\theta, W^*) < \eta.
\end{equation}
\end{theorem}
}

A detailed proof of Theorem~\ref{thm:main_cut_distance_bound}, along with necessary definitions and supporting lemmas, can be found in Appendix~\ref{app:proof_theo}.
This result demonstrates that if the INR can accurately approximate the empirical motif densities (Assumption~\ref{as:nn_perf}), and if enough data (characterized by $P$ and $n$) is available to ensure the empirical motif densities are close to the true graphon motifs (Lemma~\ref{lem:empirical_concentration_main}), then the estimated graphon is likely to be close to the true graphon in cut distance.

Although condition~\eqref{eq:sampling_cond} may seem restrictive, note that (i) it decays exponentially with the number of graphs $P$ and their size $n$ considered, so it can be made arbitrarily small by considering larger datasets and (ii) although it increases with $k$ (and therefore with $N_k$), the size of the subgraphs considered $k$ is usually small (up to 5 nodes at most).
\mycomment{
\sms{Does it really increase with $k$? if you expand the square in the exponent and put the $1/k^2$ factor in ... do you have an increasing function with $k$?} \vt{$N_k$ grows super-exponentially with $k$ I think as the number of possible non-isomorphic graphs explodes, so I believe so.}
}

\section{Moment Mixup}
\label{sec:moment_mixup}

Data augmentation is a crucial technique in machine learning, particularly in domains like graph learning, where labeled data can be scarce or expensive to obtain~\cite{ding2022dataaugmentationdeepgraph}. 
By synthetically expanding the training dataset with new, plausible examples, data augmentation helps to improve model generalization, reduce overfitting, and enhance robustness. 
In the context of graph learning, developing effective augmentation strategies is challenging due to the complex, non-Euclidean nature of graph data, where direct analogies to image or text augmentation methods are not always feasible.

% One such approach developed for graph data augmentation is \textbf{g-Mixup}. This method addresses the challenge of mixing graphs, which often have different sizes and no inherent node alignment, by operating in the space of graphons. \textbf{g-Mixup} operates by first estimating distinct graphons for each class represented in the training dataset. Subsequently, it creates novel, synthetic graph samples by performing a convex combination of these estimated class-specific graphons. New graphs are then sampled from these resulting mixed graphons. These generated samples serve as intermediate examples, falling between the original classes, and are used to augment the training data by introducing a blend of structural properties from the parent classes.

In this section, we introduce MomentMixup, a novel approach for data augmentation in graph learning. The process begins by generating novel moment profiles through convex combinations of moment vectors, where each vector $\mathbf{m}_k$ is derived from sampled graphs belonging to a distinct graph class (e.g., $\mathbf{m}_{\text{new}} = \sum \alpha_k \mathbf{m}_k$, with $\alpha_k \ge 0, \sum \alpha_k = 1$).
This interpolated moment vector, $\mathbf{m}_{\text{new}}$, is then used as the input to MomentNet, which subsequently defines a new graphon distribution, $W_{\text{new}}(\eta_i,\eta_j)$, consistent with these synthesized moments.
Finally, new graphs are sampled from $W_{\text{new}}(\eta_i,\eta_j)$ and integrated into the training set. The pseudocode of MomentMixup is provided in Algorithm 1 in Appendix F.

\begin{proposition}\label{prop:props1}
    A convex combination of graphons is not equivalent to the corresponding convex combination of their vectors of moments, with the exception of the edge density moment.
\end{proposition}
Proof of proposition~\ref{prop:props1} using a counterexample is provided in Appendix~\ref{app:proof_prop1}.
MomentMixup is developed based on the key insight from Proposition~\ref{prop:props1}. This foundational understanding distinguishes MomentMixup and offers it as an alternative to existing methods like G-Mixup~\citep{han2022gmixup_icml}. The core intuition underpinning MomentMixup is that newly generated graph samples should exhibit clear structural proximity to a specific class (i.e., similar motif counts), thereby ensuring the augmented data reinforces class-specific structural characteristics. We contend that this particular intuition, that a generated sample is structurally close to a target class, may not always be reliably achieved through G-Mixup's graphon interpolation strategy because of Proposition~\ref{prop:props1}.
\mycomment{
\sms{It is very hard to understand what this proposition means by just reading its statement. Maybe you can put some math in here? maybe relating something to the notation in the mixup section in the background? Maybe you can introduce first the algorithm and then introduce the proposition once you have the notation already laid out.} \rr{Does it sound better this way?}
}

\section{Numerical Experiments}
\label{sec:exps}
In this section, we evaluate the performance of \textbf{MomentNet} and \textbf{MomentMixup} using various synthetic and real-world datasets widely used in the literature.
The primary deep learning components of our experiments were executed on an Nvidia A100 GPU. Separately, empirical graph moments were computed using the ORCA toolkit~\cite{orca}, with its execution parallelized across an AMD EPYC 7742 64-Core Processor.

\subsection{MomentNet Evaluation}

To comprehensively evaluate our proposed \textbf{MomentNet}, we focus on two critical dimensions. First, we examine its effectiveness in the primary task of graphon estimation, determining how accurately it can capture the underlying distribution of graphs. 
Second, we address the practical applicability of our model by testing its scalability. 
This involves assessing its performance and runtime when applied to both large graphs (high number of nodes) and collections of smaller graphs, which are crucial considerations for real-world applications. 
We use $L=20000$, which is the number of samples to compute the density of moments of MomentNet using Eq.~\ref{eq:mc_moment} in both experiments.

%%%%%%%%%%%%%%%%%%%%%% Figure of MomentNet Evaulation %%%%%%%%%%%%

\begin{figure}[ht] % Placement options: h-here, t-top, b-bottom, p-page
    \centering % Center the entire figure content

    % First row: one plot
    \begin{subfigure}[b]{\textwidth} % You can adjust the width, e.g., 0.9\textwidth
        \centering
        \includegraphics[width=0.8\linewidth]{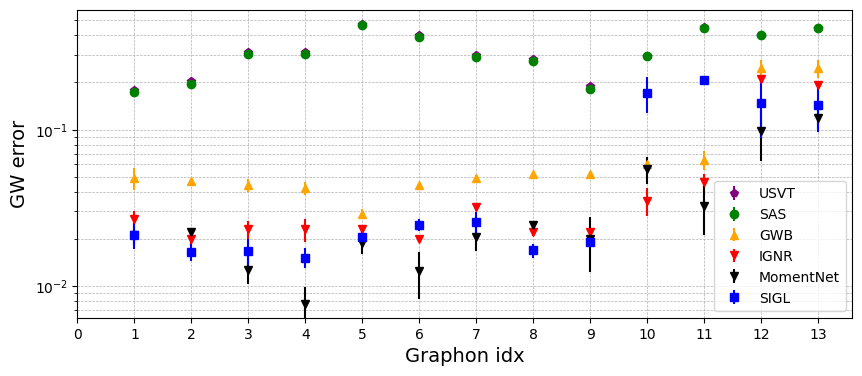} % Adjust width of image if needed
        \subcaption{Performance comparison of MomentNet compared to other graphon estimation approaches.}
        \label{fig:graphon_comparison} % Label for referencing
    \end{subfigure}

    \vspace{1em} % Adds a small vertical space between the rows of subfigures

    % Second row: two plots side by side
    \begin{subfigure}[b]{0.48\textwidth} % Adjust width as needed, ensures they fit side-by-side
        \centering
        \includegraphics[width=\linewidth]{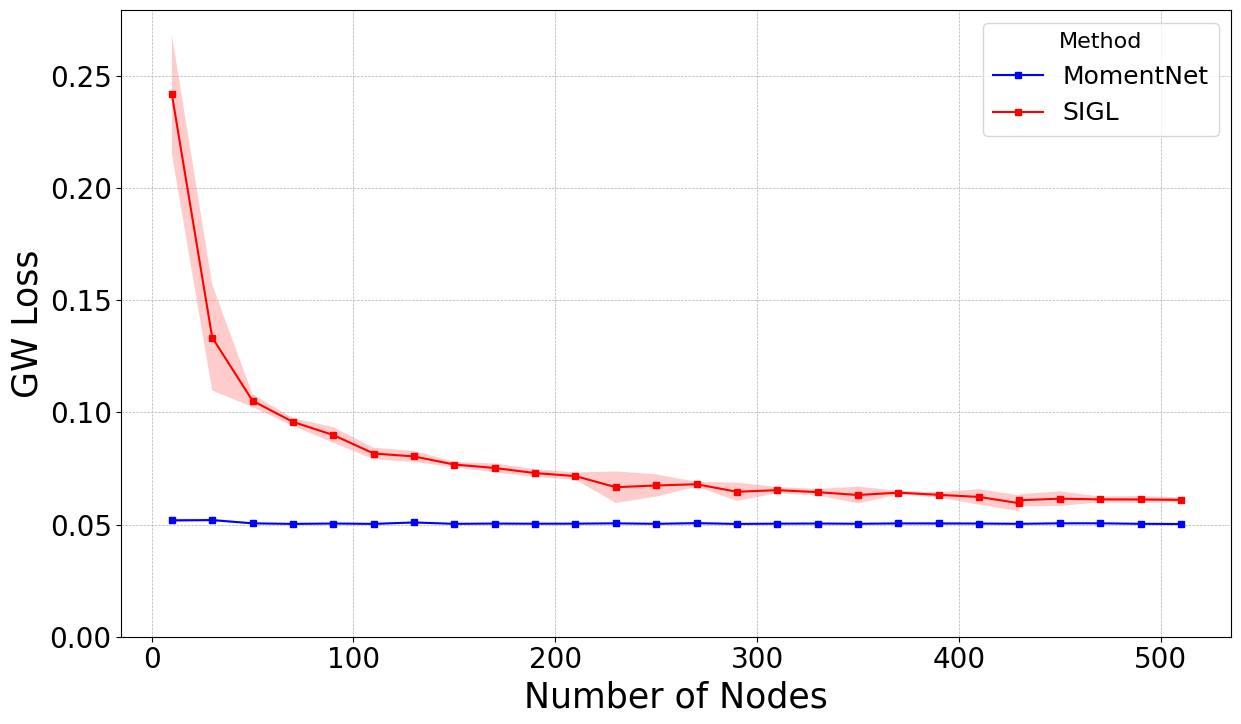} % Image takes full width of its subfigure
        \subcaption{Comparison of performance scalability of MomentNet with SIGL.}
        \label{fig:performance_scalability} % Label for referencing
    \end{subfigure}
    \hfill % Automatically distributes horizontal space between the two subfigures
    \begin{subfigure}[b]{0.48\textwidth} % Adjust width as needed
        \centering
        \includegraphics[width=\linewidth]{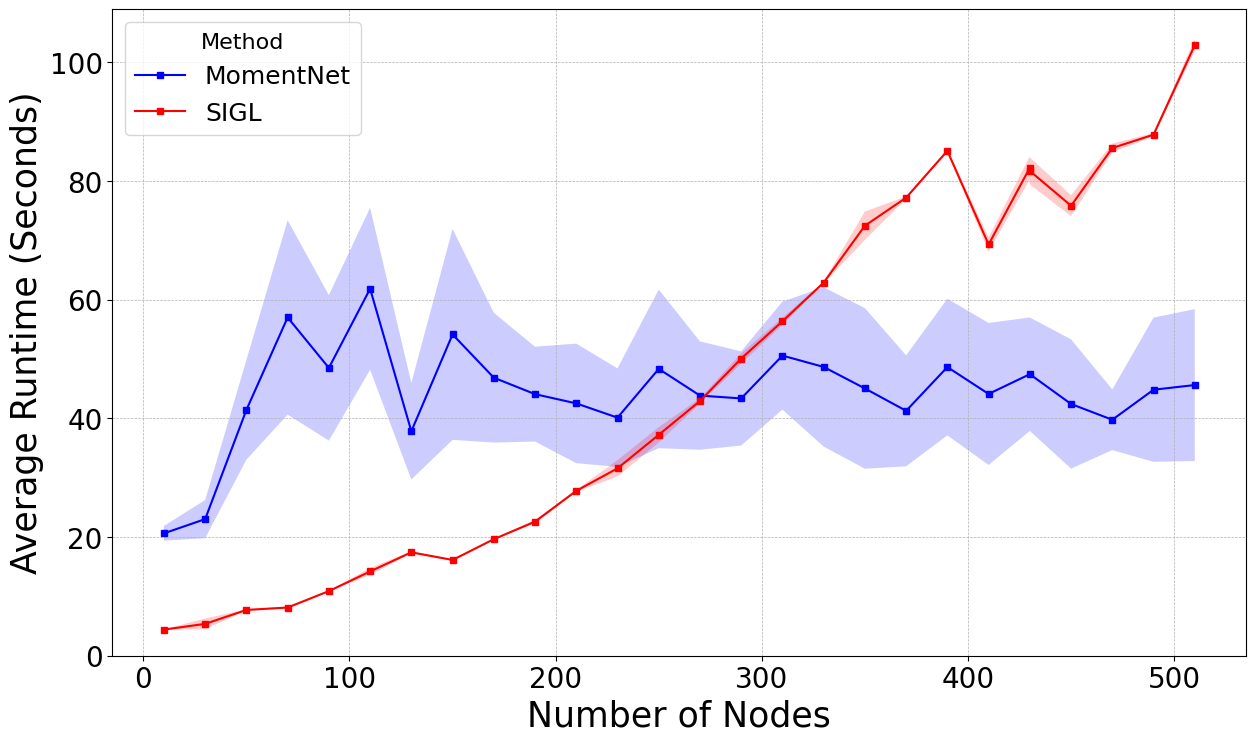} % Image takes full width of its subfigure
        \subcaption{Comparison of runtime scalability of MomentNet with SIGL.}
        \label{fig:runtime_scalability} % Label for referencing
    \end{subfigure}

    % Optional: A main caption for the entire figure
    \caption{Overall comparison of MomentNet performance and scalability.}    \label{fig:overall_momentnet_comparison}
\end{figure}

%%%%%%%%%%%%%%%%%%%%%%%%%%%%%%%%%%%%%%%%%%%%%%%%%%

\subsubsection{Graphon Estimation}

We use the 13 graphon distributions used in the literature~\cite{azizpour2025scalable, xia2023implicit}. The list of graphon distributions with their plot is provided in Appendix~\ref{app:graphon_table}.
To build our experimental dataset, we adopt the graph generation approach utilized in~\cite{xia2023implicit}.
From each graphon, we then generate 10 distinct graphs of varying sizes, specifically containing $\{75, 100, \hdots, 275, 300\}$ nodes respectively.
We utilize a one-layer MLP with 64 neurons as the INR to estimate the graphon.
For comparison, following INR training, we generate the graphon using $1000$ uniformly sampled equidistant latent variables over the interval $[0,1]$. 
The GW distance~\cite{gwb} is then computed between this estimate and the ground truth graphon.
For our method, we implemented the same steps described in section~\ref{sec:method}, considering the motifs provided in Fig. 4.
To evaluate the performance of our graphon estimation, we benchmark it against several established baseline methods. 
These include universal singular value thresholding (USVT)~\cite{chatterjee2015matrix}, sorting-and-smoothing (SAS)~\cite{xing2014consistent}, implicit graph neural representation (IGNR)~\cite{xia2023implicit}, Graph-Wasserstein barycenters (GWB)~\cite{xu2021learning}, and scalable implicit graphon learning (SIGL). 
For a consistent comparison, graphons estimated by the IGNR and SIGL baselines are sampled at a resolution of 1000, mirroring our own evaluation protocol.
Furthermore, for SAS and USVT, we zero-pad the adjacency matrices of the observed graphs to this target resolution of 1000 before their respective graphon estimation procedures are applied; this input processing strategy is similar to those employed in \cite{azizpour2025scalable, xia2023implicit}. 

The results are provided in Fig.~\ref{fig:graphon_comparison}. 
Based on the GW loss, our method outperforms the scalable state-of-the-art approach in 9 out of 13 graphons. 
Notably, our approach achieved superior results for graphons 10 and 11, where the state-of-the-art baseline (SIGL) struggled more. This difficulty might stem from the specific structures of graphons 10 and 11, which can challenge SIGL's reliance on accurately learning latent variables for its GNN-based node ordering and subsequent graphon estimation. 
Alongside the GW loss comparison, we assessed our graphon estimation via centrality measures, and the findings, detailed in Appendix~\ref{app:cent_measures}, affirmed our method's performance.

\subsubsection{Scalability Evaluation}

For experimental evaluations, we use the graphon $W(\eta_i, \eta_j) = 0.5 + 0.1 \cos(\pi \eta_i) \cos(\pi \eta_j)$ (Figure~\ref{fig:MomentNet}) for generating graph instances across multiple independent realizations for each node size $n \in \{10, \ldots, 510\}$. In each realization, 10 graphs of size $n$ are generated; MomentNet's target motif counts are averaged from these, while SIGL processes them according to its methodology. Reported performance metrics are averaged over these realizations, allowing methods to leverage a comprehensive set of samples.

The estimation error results (Figure~\ref{fig:performance_scalability}) show that MomentNet achieves strong performance, with error decreasing as $n$ increases. By leveraging multiple graph instances, MomentNet demonstrates near-optimal performance even on relatively small graphs, attributed to more accurate motif density estimation, aligning with theory.
In contrast, SIGL's error, while node-dependent, does not substantially improve from multiple graph instances, offering only slight gains for small graphs, resulting in inferior overall performance. A potential explanation is SIGL's reliance on accurate latent variable estimation. The specific graphon $W(\eta_i, \eta_j) = 0.5 + 0.1 \cos(\pi \eta_i) \cos(\pi \eta_j)$ (Figure~\ref{fig:MomentNet}) makes this challenging, as its construction ensures edge probabilities near $0.5$, leading to high variance ($W(1-W)$ near maximum). This high variance can obscure latent structure. MomentNet's averaging directly reduces density estimate variance. However, for SIGL, if each of the 10 graphs individually fails to resolve latent positions due to high variance, more such graphs may not overcome this limitation as effectively as methods that directly average structural features.

Regarding runtime (Figure~\ref{fig:runtime_scalability}), MomentNet's average runtime, despite variance, scales more favorably with increasing nodes compared to SIGL, showing a clear advantage for $n > 300$. The variance in MomentNet's runtime is due to its early stopping criteria (see Appendix~\ref{app:complexity} for detailed complexity analysis). Further experimental results are provided in Appendix~\ref{app:scale_exp}.

\mycomment{
For the experimental evaluations, we use the graphon $W(\eta_i, \eta_j) = 0.5 + 0.1 \cos(\pi \eta_i) \cos(\pi \eta_j)$ for generating graph instances; this graphon is visualized in Figure~\ref{fig:MomentNet}. We perform multiple independent realizations for each node size $n \in \{10, 20, \ldots, 510\}$. In each realization, 10 graphs of size $n$ are generated from this graphon. For MomentNet, the target motif counts for estimation are obtained by averaging the individual motif counts from these 10 graphs. SIGL also processes these 10 graphs according to its methodology. The performance metrics reported, such as estimation error and runtime, are averaged over these independent realizations. This approach allows the estimation methods to leverage a more comprehensive set of samples from the underlying graphon in each estimation run.

The experiment results for estimation error are provided in Figure~\ref{fig:performance_scalability}. MomentNet achieves strong performance, with its error decreasing as the number of nodes increases. By leveraging multiple graph instances, MomentNet demonstrates near-optimal performance even on relatively small graphs. This improvement is attributed to a more accurate estimation of motif densities, which aligns with the theoretical intuition that a larger sample size provided by both increasing graph size and the number of graph instances leads to improved performance for our method.

In contrast, while SIGL's error also shows a dependency on the node count, its performance does not gain substantial improvement from leveraging multiple graph instances; it offers only slight gains, primarily for small graphs. Consequently, SIGL's overall performance is inferior to that of MomentNet. A potential explanation for SIGL's limited improvement from multiple graphs is that its performance primarily relies on the accurate estimation of latent variables. The specific graphon $W(\eta_i, \eta_j) = 0.5 + 0.1 \cos(\pi \eta_i) \cos(\pi \eta_j)$ utilized in our experiments (see Figure~\ref{fig:MomentNet}) is an instance where this challenge is pronounced. Its construction ensures edge probabilities are consistently near $0.5$, leading to high variance in the edge generation process (where $W(1-W)$ approaches its maximum). Such high variance can obscure the underlying latent structure, thereby complicating the recovery of latent variables. Averaging motif counts, as done by MomentNet, directly reduces the variance of the density estimates. However, for SIGL, if each of the 10 graphs (of size $n$) generated from this high-variance graphon is individually insufficient to clearly resolve latent positions, then simply having more such graphs might not overcome this fundamental limitation as effectively as it helps methods directly averaging structural features. In such scenarios, merely having multiple graph instances might not sufficiently aid SIGL if the core difficulty lies in resolving latent positions from inherently noisy observations from each graph.

Regarding the runtime, Figure~\ref{fig:runtime_scalability} reveals that the average runtime of MomentNet, while exhibiting variance, scales more favorably with an increasing number of nodes compared to SIGL. Specifically, MomentNet demonstrates a clear runtime advantage over SIGL as the number of nodes becomes large (exceeding approximately $300$ in the depicted data). The notable variance in runtime for MomentNet, represented by the shaded areas, is attributed to the early stopping criteria employed during its training process. A more detailed analysis of computational complexity is provided in Appendix~\ref{app:complexity}.
}

\subsection{MomentMixup Evaluation}
To evaluate the performance of our \textbf{MomentMixup} framework, we conducted graph classification experiments on several real-world datasets: AIDS~\cite{AIDS}, IMDB-Binary~\cite{datasets}, IMDB-Multi~\cite{datasets}, and Reddit-Binary~\cite{datasets}. Detailed descriptions of these datasets are provided in Appendix~\ref{app:mixup_details}. To ensure a fair comparison with prior work, we adopted the same data splitting methodology as reported in previous literature~\cite{azizpour2025scalable, han2022gmixup_icml}.
For data augmentation, we treated $\alpha_{mix}, N_{\text{nodes}}, N_{\text{graph}}$, and $N_{\text{sample}}$ as hyperparameters in Algorithm 1. We employ the GIN architecture~\cite{GIN} as the graph classification model. 

Table~\ref{tab:mmixup_results} presents the model's accuracy on the test set.
The results demonstrate that our method achieves a better performance gain over the standard G-Mixup approach on three datasets.
%This superiority can be attributed to our method's enhanced capability in preserving the structural features of graphs~\ref{prop:props1}.
%
% However, our performance on the Reddit-Binary dataset did not surpass the current state-of-the-art.
% This was primarily because, for the experiment, we restricted feature preservation by utilizing only nine motifs.
% It is worth noting that on this particular dataset, characterized by large graphs, the SIGL method tends to perform well.
% Conversely, when dealing with datasets composed of smaller graphs, such as the AIDS dataset, SIGL's performance is comparatively poor, whereas our method excels.
As highlighted in the previous section, our method demonstrates a distinct advantage on datasets composed of smaller graphs, such as AIDS, where it notably outperforms techniques like SIGL. While our results on the Reddit-Binary dataset, which features very large graphs and where SIGL performs strongly, were influenced by the experimental choice of using a limited set of nine motifs, this contrast further illuminates a key insight: the optimal choice of mixup method can be highly dependent on graph characteristics, particularly size. Our approach appears particularly well-suited for capturing structural nuances in smaller graphs where fewer motifs can still provide rich representative information.

\begin{table*}[!htb]
    \centering
    \caption{\small{Classification accuracy of G-Mixup, MomentMixup, and baselines on different datasets.}}
    \label{tab:mmixup_results}
    \begin{tabular}{lcccc}
        \hline
        \textbf{Dataset} & \textbf{IMDB-B} & \textbf{IMDB-M} & \textbf{REDD-B} & \textbf{AIDS} \\
        \hline
        \#graphs     & 1000 & 1500 & 2000 & 2000  \\
        \#classes    & 2 & 3 & 2 & 2 \\
        \#avg.nodes  & 19.77 & 13.00 & 429.63 & 15.69 \\
        \#avg.edges  & 96.53 & 65.94 & 497.75 & 16.2  \\
        \hline
        \multicolumn{5}{c}{\textbf{GIN}}\\
        \hline
        No Augmentation                & 71.55$\pm$3.53 & 48.83$\pm$2.75 & 91.78$\pm$1.09  & 98$\pm$1.2 \\
        G-Mixup {w/ USVT}     & 71.94$\pm$3.00 & 50.46$\pm$1.49 & 91.32$\pm$1.51  & 97.8$\pm$0.9 \\
        G-Mixup {w/ SIGL}     & 73.95$\pm$2.64 & 50.70$\pm$1.41 & \textbf{92.25}$\pm$1.41  & 97.3$\pm$1 \\
        MomentMixup           & \textbf{74.30}$\pm$2.70 & \textbf{50.95}$\pm$1.93  & 91.8 $\pm$ 1.2 & \textbf{98.5}$\pm$0.6 \\
        \hline
    \end{tabular}
\end{table*}

\section{Conclusions}
\label{sec:conclusion}

In this paper, we introduced a novel, scalable graphon estimator leveraging INRs via direct moment matching, called MomentNet.
This approach bypasses latent variables and costly GW optimizations, offering a theoretically grounded, polynomial-time solution for estimating graphons from empirical subgraph counts, with provable guarantees on its accuracy.
We further proposed MomentMixup, a new data augmentation technique that performs mixup in the moment space, then obtains the estimated graphon using MomentNet, and finally samples new graphs from this graphon.
Our empirical results validate the effectiveness of our estimator, demonstrating superior or comparable performance against state-of-the-art methods in graphon estimation benchmarks, and show that MomentMixup improves graph classification accuracy by generating structurally meaningful augmented data.

%Our work provides efficient and robust tools for network analysis and graph learning. 
Despite its strengths, our method's reliance on a pre-selected set of moments for graphon estimation is a limitation; performance can degrade if these moments are insufficient or noisy. Additionally, modeling a single graphon (per class for MomentMixup) may not capture highly heterogeneous graph data. Future work could address these by developing adaptive moment selection techniques and exploring extensions to learn mixtures of graphons. Further enhancements include adapting our moment-based approach for attributed or dynamic networks and integrating feature learning into the estimation process, broadening the applicability of our framework.

\section{Acknowledgments}
\label{sec:acks}
This research was sponsored by the National Science Foundation under award CCF-234048; by the Spanish AEI (10.13039/501100011033) under Grants PID2022-136887NBI00 and FPU20/05554; and by the Community of Madrid within the ELLIS Unit Madrid framework and the IDEA-CM (TEC-2024/COM-89) and CAM-URJC F1180 (CP2301) grants.

\bibliographystyle{apalike}
\bibliography{citations}

%%%%%%%%%%%%%%%%%%%%%%%%%%%%%%%%%%%%%%%%%%%%%%%%%%%%%%%%%%%%

\appendix

%%%%%%%%%%%%%% Details of Related Work %%%%%%%%%%%%%%

\section{Detailed Related Work} \label{app:detailed_related_work}

\paragraph{Graphon Estimation}
Graphon estimation aims to recover the underlying generative structure of observed networks. Classical approaches include methods based on histogram estimators by partitioning nodes according to degree or other structural properties~\citep{borgs2008convergent,lovasz2012large,xing2014consistent}, and fitting stochastic block models (SBMs) or their variants, which can be viewed as piecewise constant graphon estimators~\citep{airoldi2008mixed,gao2015rate}. Universal singular value thresholding (USVT)~\citep{chatterjee2015matrix} offers a non-parametric approach for estimating graphons from a single adjacency matrix, particularly effective for low-rank structures. However, many of these methods face challenges in terms of computational cost for large graphs, achieving resolution-free approximation, or may rely on specific structural assumptions (e.g., piecewise constant for SBMs).

More recently, scalable graphon estimation techniques have gained prominence. For example, some works aim at minimizing distances between graph representations but often involve computationally expensive metrics like the GW distance~\citep{peyre2016gromov,xu2019gromov,xu2021learning}, which can be a bottleneck for large networks. The advent of INRs has opened new avenues for continuous, resolution-free graphon estimation. 
For instance, IGNR (Implicit Graphon Neural Representation)~\citep{xia2023implicit} proposed to directly model graphons using neural networks, enabling the representation of graphons up to arbitrary resolutions and efficient generation of arbitrary-sized graphs. IGNR also addresses unaligned input graphs of different sizes by incorporating the Gromov-Wasserstein distance in its learning framework, often within an auto-encoder setup for graphon learning. 
Subsequently, SIGL (Scalable Implicit Graphon Learning)~\citep{azizpour2025scalable} further advanced INR-based graphon estimation by combining INRs with Graph Neural Networks (GNNs). 
In SIGL, GNNs are utilized to improve graph alignment by determining appropriate node orderings, aiming to enhance scalability and learn a continuous graphon at arbitrary resolutions, with theoretical results supporting the consistency of its estimator.
While these INR-based techniques offer significant advantages in terms of resolution-free representation and handling unaligned data, they still implicitly involve latent variable modeling or rely on GW-like objectives for alignment.
Our proposed method builds upon the representational power of INRs but distinguishes itself by directly recovering the graphon via moment matching. 
This avoids the need for latent variables, complex metric computations like GW, and provides a theoretically grounded estimation framework that naturally handles multiple observed graphs by matching aggregated empirical moments.

\paragraph{Data Augmentation for Graph Classification}
Data augmentation is crucial for improving the generalization of GNNs and other graph learning models, especially when labeled data is scarce.
%Common augmentation techniques in the graph domain include structural modifications like edge dropping/addition \cite{rong2019dropedge,you2020graph}, node attribute masking \cite{you2020graph}, or subgraph sampling \cite{you2020graph}. These methods operate directly on the discrete graph structure or node features.
Mixup~\citep{zhang2017mixup}, which creates synthetic examples by linearly interpolating pairs of samples and their labels, has shown remarkable success in various domains. 
Its adaptation to graph data has been explored through several avenues, addressing challenges such as varying node counts, lack of alignment, and the non-Euclidean nature of graphs. 
For instance, \citet{wang2021mixup} proposed interpolating hidden states of GNNs.
Particularly relevant to our work are G-Mixup~\citep{han2022gmixup_icml} and GraphMAD~\citet{navarro2023graphmad}, which recognize the difficulties of direct graph interpolation and propose to augment graphs for graph classification by operating in the space of graphons.
GraphMAD~\citet{navarro2023graphmad} projects graphs into the latent space of graphons and implements nonlinear mixup strategies like convex clustering.
G-Mixup~\citep{han2022gmixup_icml} first estimates a graphon for each class of graphs from the training data. 
Then, instead of directly manipulating discrete graph structures, G-Mixup interpolates these estimated graphons of different classes in their continuous, Euclidean representation to obtain mixed graphons. 
Synthetic graphs for augmentation are subsequently generated by sampling from these mixed graphons. 
This technique has also been adopted as an augmentation strategy in the evaluation pipelines of some graphon estimation studies for downstream tasks~\citep{azizpour2025scalable}.

%%%%%%%%%%%%%%%%%%%%%%%%%%%%%%%%%%%%%%%%%%%%%%%%%%%%%%%%%

%%%%%%%%%%%%%%%%%%%%%%%%%%%%%%%%%%%%%%%%%%%%%%%%%%%%%
\section{Proof of Theorem~\ref{thm:main_cut_distance_bound}}\label{app:proof_theo}

\subsection{Supporting Lemmas}

We rely on the following established and derived results. Lemma~\ref{lem:empirical_concentration_main} is an original contribution of this work, while Lemma~\ref{lem:motif_to_cut_distance} is Theorem 3.7 (b) in~\citet{borgs2008convergent} and it is included here for completeness.

\begin{lemma}[Concentration of Empirical Motifs]
\label{lem:empirical_concentration_main}
Let $F$ be a simple graph with $k = |\ccalV_F|$ vertices. For $P \ge 1$ graphs $G_1, \ldots, G_P$, each sampled independently from $G_n (W^*)$, and for any error tolerance $\epsilon_s > 0$, the probability that the empirical motif density $\bar{t} (F,W^*) = \frac{1}{P} \sum_{p=1}^P t(F,G_p)$ deviates from the true motif density $t(F, W^*)$ is bounded as
%The difference between the expected empirical density and the true density is known to have magnitude $|\mathbb{E}[t(F,G_n (W))] - t(F,W)| \approx \frac{k(k-1)}{2n}$. Specifically, if $\epsilon_s > \frac{k(k-1)}{2n}$, then
\begin{equation}
\mathbb{P}[|\bar{t} (F,W^*) - t(F,W^*)| \geq \epsilon_s] \leq 2 \exp\left(-\frac{Pn}{4k^2} \left(\epsilon_s - \frac{k(k-1)}{2n}\right)^2 \right),
\end{equation}
for $\epsilon_s > \frac{k(k-1)}{2n}$.
\end{lemma}

\begin{proof}
Let $X_p = t(F, G_p)$ for $p=1, \ldots, P$. The graphs $G_p$ are independent samples from $G_n (W^*)$, so the random variables $X_p$ are independent and identically distributed.

We leverage concentration properties of $t(F, G_n (W^*))$ in~\citet[Lemma 4.4]{borgs2008convergent}, stating that $t(F, G_n (W^*))$ is concentrated around $t(F,W^*)$ with probability
\begin{equation}
    \mathbb{P}[|t(F,G_n (W^*)) - t(F,W^*)| > \delta] \leq 2 \exp(-n\delta^2/(4k^2)).
    \label{eq:borgs_concentration}
\end{equation}
This implies that the variable $Z = t(F,G_n (W^*)) - t(F,W^*)$ behaves like a sub-Gaussian random variable~\citep{van2014probability}\footnote{A random variable $Y$ is $\sigma^2$-sub-Gaussian if $\mathbb{E}[e^{\lambda Y}] \le e^{\lambda^2 \sigma^2/2}$ for all $\lambda \in \mathbb{R}$, which implies the tail bound $\mathbb{P}[|Y| \ge \delta] \le 2 e^{-\delta^2/(2\sigma^2)}$.}. Comparing the exponent $-\frac{n\delta^2}{4k^2}$ from~\eqref{eq:borgs_concentration} with the sub-Gaussian tail exponent $-\frac{\delta^2}{2\sigma^2}$, we see that $t(F,G_n (W^*)) - t(F,W^*)$ is sub-Gaussian with parameter $\sigma_Z^2 = \frac{2k^2}{n}$.

The variables we are averaging are $X_p = t(F,G_p)$ with $G_p \sim G_n (W^*)$. Let $\mu_n = \mathbb{E}[X_p] = \mathbb{E}[t(F, G_n (W^*))]$. The centered variables fulfill $X_p - \mu_n = (t(F,G_p) - t(F,W^*)) - (\mathbb{E}[t(F,G_p)] - t(F,W^*))$. Subtracting a constant (the bias $\mathbb{E}[t(F,G_p)] - t(F,W^*)$) from a sub-Gaussian variable preserves its sub-Gaussian property with the same parameter. Thus, $X_p - \mu_n$ are independent, zero-mean, and $\sigma^2$-sub-Gaussian with $\sigma^2 = \sigma_Z^2 = \frac{2k^2}{n}$.

The average of $P$ independent $\sigma^2$-sub-Gaussian random variables is $(\sigma^2/P)$-sub-Gaussian~\citep{van2014probability}. Let $\bar{Y} = \frac{1}{P}\sum_{p=1}^P (X_p - \mu_n) = \bar{t} (F,W^*) - \mu_n$. Then $\bar{Y}$ is $\left(\frac{2k^2}{nP}\right)$-sub-Gaussian.
The tail bound for $\bar{Y}$ is
\begin{equation}
\mathbb{P}[|\bar{Y}| \geq \delta] \leq 2 \exp\left(-\frac{\delta^2}{2 \cdot \frac{2k^2}{nP}}\right) = 2 \exp\left(-\frac{\delta^2 nP}{4k^2}\right).
\label{eq:subgaussian_average_bound_appendix}
\end{equation}
Substituting $\bar{Y} = \bar{t} (F, W^*) - \mu_n$, we get the concentration bound for the empirical mean around the expected mean:
\begin{equation}
\mathbb{P}[|\bar{t} (F,W^*) - \mu_n| \geq \delta] \leq 2 \exp\left(-\frac{\delta^2 nP}{4k^2}\right).
\label{eq:concentration_around_mean_appendix}
\end{equation}
We are interested in the deviation of $\bar{t} (F,W^*)$ from the true motif density $t(F,W^*)$. We use the triangle inequality to relate this deviation to the deviation from the mean $\mu_n$
\begin{equation}
|\bar{t} (F,W^*) - t(F,W^*)| \leq |\bar{t} (F,W^*) - \mu_n| + |\mu_n - t(F,W^*)|.
\end{equation}
Let $B_n = |\mu_n - t(F,W^*)|$ be the bias of the empirical estimate. It is known from the theory of graph limits (e.g., related to~\citet[Lemma 4.3]{borgs2008convergent}) that this bias is bounded by $B_n \leq \frac{k(k-1)}{2n}$.
If the deviation from the true density is at least $\epsilon_s$, i.e., $|\bar{t} (F,W^*) - t(F,W)| \geq \epsilon_s$, then it must be that $|\bar{t} (F,W^*) - \mu_n| \geq \epsilon_s - B_n$. This implication requires $\epsilon_s > B_n$ for the bound to be meaningful.
Thus, for $\epsilon_s > B_n$
\begin{equation}
\mathbb{P}[|\bar{t} (F,W^*) - t(F,W^*)| \geq \epsilon_s] \leq \mathbb{P}[|\bar{t} (F,W^*) - \mu_n| \geq \epsilon_s - B_n].
\end{equation}
Using the inequality~\eqref{eq:concentration_around_mean_appendix} with $\delta = \epsilon_s - B_n$
\begin{equation}
\mathbb{P}[|\bar{t} (F,W^*) - t(F,W^*)| \geq \epsilon_s] \leq 2 \exp\left(-\frac{(\epsilon_s - B_n)^2 nP}{4k^2}\right).
\end{equation}
Introducing the upper bound for the bias, $B_n \leq \frac{k(k-1)}{2n}$
\begin{align}
\mathbb{P}[|\bar{t} (F,W^*) - t(F,W^*)| \geq \epsilon_s] &\leq 2 \exp\left(-\frac{\left(\epsilon_s - \frac{k(k-1)}{2n}\right)^2 nP}{4k^2}\right) \\
&= 2 \exp\left(-\frac{Pn}{4k^2} \left(\epsilon_s - \frac{k(k-1)}{2n}\right)^2 \right).
\end{align}
This bound is valid when $\epsilon_s > \frac{k(k-1)}{2n}$, as required by the lemma statement.
\end{proof}

\begin{lemma}[Motif Proximity Implies Cut Distance Proximity (\citet{borgs2008convergent}, Theorem 3.7 (b))]
\label{lem:motif_to_cut_distance}
For any integer $k \ge 1$, if the motif distance between two graphons $W_1$ and $W_2$ fulfills $|t(F,W_1) - t(F,W_2)| < \delta_M = 3^{-k^2}$ for every simple graph $F \in \mathcal{F}_k$, then the cut distance between $W_1$ and $W_2$ is upper bounded by
\begin{equation}
    d_{cut} (W_1, W_2) \leq \eta = \frac{22C}{\sqrt{\log_2 k}},
\end{equation}
where $C = \max \{ 1, \| W_1 \|_{\infty}, \| W_2 \|_{\infty} \}$.
\end{lemma}

\begin{assumption}[Neural Network Approximation Capability]
Given a sufficiently expressive neural network architecture, it can be trained to find parameters $\theta$ such that for any set of empirical motif densities $\{\bar{t} (F, W^*)\}_{F \in \mathcal{F}_k}$ computed from data, and any desired approximation error $\epsilon_a > 0$, the neural network's MC motif estimates $\hat{t}_\theta(F, \hat{W}_\theta)$ satisfy
\begin{equation}
|\hat{t}_\theta (F, \hat{W}_\theta) - \bar{t} (F, W^*)| < \epsilon_a,
\end{equation}
for all $F \in \mathcal{F}_k$.
\label{as:nn_perf}
\end{assumption}

This assumption is fundamentally supported by the Universal Approximation Theorem (UAT)~\citep{cybenko1989approximation, hornik1989multilayer, leshno1993multilayer}. The UAT posits that a neural network with sufficient capacity (e.g., an adequate number of neurons in one or more hidden layers and appropriate non-linear activation functions) can approximate any continuous function to an arbitrary degree of accuracy on a compact domain. In our context, the INR $f_\theta$ models the graphon $W: [0,1]^2 \to [0,1]$. The motif density $t(F, W)$ (as defined in Equation~\ref{eq:motif_density}) is a continuous functional of $W$, meaning small changes in $W$ lead to small changes in $t(F,W)$. Consequently, if the INR $f_\theta$ can approximate any continuous graphon function, it can learn a specific $f_\theta$ such that the motif densities of the graphon estimated by the INR $t(F, f_\theta)$ are arbitrarily close to some target values. Given that our estimated motif densities $\hat{t}_\theta(F, \hat{W}_\theta)$ are Monte Carlo approximations of $t(F, f_\theta)$, they too can approach these target values (the empirical densities $\bar{t} (F, W^*)$) as the approximation of the underlying function by $f_\theta$ improves as the number of Monte Carlo samples $L$ increases. The assumption thus relies on the INR's capacity to learn a suitable graphon function $f_\theta$ and the optimization process's ability to find the parameters $\theta$ that make the resulting motif estimates $\hat{t}_\theta(F, \hat{W}_\theta)$ match the empirical observations $\bar{t} (F, W^*)$.

\subsection{Proof of Theorem~\ref{thm:main_cut_distance_bound}}

\begin{proof}[Proof of Theorem~\ref{thm:main_cut_distance_bound}]
Our goal is to bound the cut distance $d_{\text{cut}}(\hat{W}_\theta, W^*)$ by $\eta$, which is achieved if we can show that $|\hat{t}(F,\hat{W}_\theta) - t(F,W^*)| < \delta_M$ for all simple graphs $F$ with $|\ccalV_F|=k$ and where the values of both $\eta$ and $\delta_M$ are provided in Lemma~\ref{lem:motif_to_cut_distance}.

Consider any graph $F \in \mathcal{F}_k$. Using the triangle inequality, we can bound the difference between the neural network's motif estimate and the true graphon motif
\begin{equation}
|\hat{t}_\theta(F,\hat{W}_\theta) - t(F, W^*)| \leq |\hat{t}_\theta(F,\hat{W}_\theta) - \bar{t} (F, W^*)| + |\bar{t} (F, W^*) - t(F,W^*)|.
\label{eq:triang_motif_distance}
\end{equation}

By Assumption~\ref{as:nn_perf} on the neural network's training performance, we guarantee
\begin{equation}
|\hat{t}_\theta(F, \hat{W}_\theta) - \bar{t} (F,W^*)| < \epsilon_a = \frac{\delta_M}{2},
\end{equation}
for every $F \in \mathcal{F}_k$.

Now we need to bound the second term in the right-hand side of~\eqref{eq:triang_motif_distance}, the deviation of the empirical motif density from the true motif density $|\bar{t} (F,W^*) - t(F,W^*)|$. We use Lemma~\ref{lem:empirical_concentration_main} with the sampling error tolerance set to $\epsilon_s = \frac{\delta_M}{2}$. For this lemma to apply, we require $\epsilon_s > \frac{k(k-1)}{2n}$, which is equivalent to $\frac{\delta_M}{2} > \frac{k(k-1)}{2n}$, or $n > \frac{k(k-1)}{\delta_M}$. This condition is enforced in the theorem statement.

For a \textit{specific} graph $F \in \mathcal{F}_k$, the probability that the sampling error is large is bounded by Lemma~\ref{lem:empirical_concentration_main}
\begin{equation}
\mathbb{P}\left[|\bar{t} (F,W^*) - t(F,W^*)| \geq \frac{\delta_M}{2}\right] \leq 2 \exp\left(-\frac{Pn}{4k^2} \left(\frac{\delta_M}{2} - \frac{k(k-1)}{2n}\right)^2 \right).
\end{equation}
Let $\mathbb{P}_{\text{fail},F}$ denote this upper bound for a single graph $F \in \ccalF_k$.
However, we require the sampling error $|\bar{t} (F,W^*) - t(F,W^*)|$ to be less than $\frac{\delta_M}{2}$ for \textit{all} graphs $F \in \mathcal{F}_k$ simultaneously. By the union bound, the probability that there exists at least one graph $F \in \mathcal{F}_k$ for which the sampling error is $\frac{\delta_M}{2}$ or more is at most the sum of the probabilities for each individual graph
\begin{equation}
\mathbb{P}[\exists F \in \mathcal{F}_k \text{ s.t. } |\bar{t} (F,W^*) - t(F,W^*)| \geq \frac{\delta_M}{2}] \leq \sum_{F \in \mathcal{F}_k} \mathbb{P}_{\text{fail},F}.
\end{equation}
Since $|\ccalV_F|=k$ for all $F \in \mathcal{F}_k$, the bound $\mathbb{P}_{\text{fail},F}$ is identical for all these graphs. The sum is thus $N_k \cdot \mathbb{P}_{\text{fail},F}$, where we recall that $N_k = | \ccalF_k |$.
The condition \eqref{eq:sampling_cond} in the theorem is precisely set to ensure that this total probability of failure is less than the desired confidence level $\zeta$
\begin{equation}
N_k \cdot 2 \exp\left(-\frac{Pn}{4k^2}\left(\frac{\delta_M}{2} - \frac{k(k-1)}{2n}\right)^2\right) < \zeta.
\end{equation}
Therefore, with probability at least $1-\zeta$ (over the random graph samples $G_p$), the event that $|\bar{t} (F,W^*) - t(F,W^*)| < \frac{\delta_M}{2}$ holds for all $F \in \mathcal{F}_k$ occurs.

Conditioned on this high-probability event, and using the neural network approximation in Assumption \ref{as:nn_perf}, we have for every $F \in \ccalF_k$
\begin{equation}
|\hat{t}_{\theta} (F,\hat{W}_\theta) - t(F,W^*)| \leq |\hat{t}_\theta(F,\hat{W}_\theta) - \bar{t} (F,W^*)| + |\bar{t} (F,W^*) - t(F,W^*)| < \frac{\delta_M}{2} + \frac{\delta_M}{2} = \delta_M.
\end{equation}
Since $|\hat{t}_{\theta} (F,\hat{W}_\theta) - t(F,W^*)| < \delta_M$ holds for all $F \in \ccalF_k$, Lemma~\ref{lem:motif_to_cut_distance} implies that the cut distance between the estimated graphon $\hat{W}_\theta$ and the true graphon $W^*$ is less than $\eta$
\begin{equation}
d_{\text{cut}}(W_\theta, W) < \eta,
\end{equation}
with probability at least $1-\zeta$, concluding the proof.
\end{proof}

%%%%%%%%%%%%%%%%%%%%%%%%%%%%%%%%%%%%%%%%%%%%%%%%%%%%%%%%%%%%

%%%%%%%%%%%%%%%%%%%%%% unbiased moment estimation %%%%%%%%%%%%

\section{Unbiasedness of Monte Carlo Estimator for an INR-Based Graphon Moment Estimator}\label{app:unbiased}

Let $F = (\mathcal{V}_F, \mathcal{E}_F)$ be a graph, where $\mathcal{V}_F$ is a set of $k = |\mathcal{V}_F|$ vertices and $\mathcal{E}_F$ is the set of edges. Let $f_{\theta}: [0,1]^2 \to [0,1]$ be an Implicit Neural Representation (INR) parameterized by $\theta$, which models the probability of an edge existing between two nodes based on their latent variables $\eta_i, \eta_j \in [0,1]$, and its estimated graphon is denoted by $\hat{W}_\theta$.

The likelihood of observing the graph structure $F$ given a specific set of latent variable assignments $\boldsymbol{\eta} = \{ \eta_v \}_{v \in \mathcal{V}_F}$ and the INR model $f_{\theta}$ is given by
\begin{equation}
    P_{\theta} (\boldsymbol{\eta}; F, \hat{W}_\theta) = \prod_{(i,j) \in \mathcal{E}_F} \hat{W}_\theta(\eta_i, \eta_j)\prod_{(i,j) \notin \mathcal{E}_F} (1-\hat{W}_\theta(\eta_i, \eta_j)).
\end{equation}
The quantity $t^{\prime}_\theta (F, \hat{W}_\theta )$ is defined as this likelihood integrated over all possible configurations of the latent variables in the $k$-dimensional unit hypercube
\begin{equation}
    t^{\prime}_\theta (F, \hat{W}_\theta ) = \int_{[0,1]^k} P_{\theta}(\boldsymbol{\eta}; F, \hat{W}_\theta ) d\boldsymbol{\eta},
    \label{eq:expr_tprime_likelihood}
\end{equation}
where $d\boldsymbol{\eta} = \prod_{v \in \mathcal{V}_F} d\eta_v$.

The $L$-sample Monte Carlo estimator for $t^{\prime}_\theta (F, \hat{W}_\theta )$ is given by
\begin{equation}
    \hat{t}^{\prime}_{\theta} (F, \hat{W}_{\theta}) = \frac{1}{L} \sum_{l=1}^{L} P(\boldsymbol{\eta}^{(l)}; F, \hat{W}_{\theta}).
    \label{eq:mc_estimator}
\end{equation}
For this estimation, each sample $\boldsymbol{\eta}^{(l)} = [\eta^{(l)}_{v_1}, \dots, \eta^{(l)}_{v_k}]$ is a vector where each component $\eta^{(l)}_{v}$ (for $v \in \mathcal{V}_F$) is drawn independently from the uniform distribution $\ccalU [0,1]$.

\subsection{Unbiasedness of the Estimator}

\begin{theorem}
The Monte Carlo estimator $\hat{t}^{\prime}_{\theta} (F, \hat{W}_{\theta})$ is an unbiased estimator of $t^{\prime}_\theta (F, \hat{W}_\theta)$.
\end{theorem}

\begin{proof}
To show that the Monte Carlo estimation $\hat{t}^{\prime}_{\theta} (F, \hat{W}_{\theta})$ is an unbiased estimator of $t^{\prime}_{\theta} (F, \hat{W}_{\theta})$, we need to prove that $\mathbb{E}[\hat{t}_{\theta}^{\prime} (F, \hat{W}_{\theta})] = t^{\prime}_\theta (F, \hat{W}_\theta )$.

The expectation of the estimator is:
\begin{align}
    \mathbb{E}[\hat{t}^{\prime}_{\theta} (F, \hat{W}_{\theta})] &= \mathbb{E}\left[\frac{1}{L} \sum_{l=1}^{L} P_{\theta}(\boldsymbol{\eta}^{(l)}; F, \hat{W}_{\theta})\right] \nonumber \\
    &= \frac{1}{L} \sum_{l=1}^{L} \mathbb{E}[P_{\theta}(\boldsymbol{\eta}^{(l)}; F, \hat{W}_{\theta})] \quad \text{(by linearity of the expectation)}. \label{eq:expectation_MC}
\end{align}
Since each sample $\boldsymbol{\eta}^{(l)}$ is drawn independently from the same uniform distribution, therefore its pdf is $p(\boldsymbol{\eta})=1$ on $[0,1]^k$, the expectation $\mathbb{E}[P_{\theta}(\boldsymbol{\eta}^{(l)}; F, \hat{W}_{\theta})]$ is the same for all $l$. Let this common expectation be $\mathbb{E}[P_{\theta}(\boldsymbol{\eta}; F, \hat{W}_{\theta})]$, whose value is
\begin{align*}
    \mathbb{E}[P_{\theta}(\boldsymbol{\eta}; F, \hat{W}_{\theta})] &= \int_{[0,1]^k} P_{\theta}(\boldsymbol{\eta}; F, \hat{W}_{\theta}) p(\boldsymbol{\eta}) d\boldsymbol{\eta} \\
    &= \int_{[0,1]^k} P_{\theta}(\boldsymbol{\eta}; F, \hat{W}_{\theta})) \cdot 1 \, d\boldsymbol{\eta} \quad (\text{since } p(\boldsymbol{\eta})=1 \text{ on } [0,1]^k) \\
    &= t^{\prime}_\theta (F, \hat{W}_\theta ),
\end{align*}
according to~\eqref{eq:expr_tprime_likelihood}. Substituting this back into~\eqref{eq:expectation_MC}
\begin{align*}
    \mathbb{E}[\hat{t}^{\prime}_{\theta} (F, \hat{W}_{\theta})] &= \frac{1}{L} \sum_{l=1}^{L} t^{\prime}_\theta (F, \hat{W}_\theta ) \\
    &= \frac{1}{L} (L \cdot t^{\prime}_\theta (F, \hat{W}_\theta )) \\
    &= t^{\prime}_\theta (F, \hat{W}_\theta ).
\end{align*}
Thus, $\mathbb{E}[\hat{t}^{\prime}_{\theta} (F, \hat{W}_{\theta})] = t^{\prime}_\theta (F, \hat{W}_\theta )$, which shows that the Monte Carlo estimator $\hat{t}^{\prime}_{\theta} (F, \hat{W}_{\theta})$ is an unbiased estimator of $t^{\prime}_\theta (F, \hat{W}_\theta )$. This means that, on average, the estimator will yield the true value of the integral defined by $f_{\theta}$ and the graph structure $F$. 
\end{proof}

%%%%%%%%%%%%%%%%%%%%%%%%%%%%%%%%%%%%%%%%%%%%%%%%%%%%%%%%%%%%%

%%%%%%%%%%%%%%%%%%%%%%%%%%%% Time complexity of MomentNet %%%%%%%%%%%%%%%%

\section{Time Complexity of \textsc{MomentNet}}\label{app:complexity}

\paragraph{Stage 1: parallel motif–density extraction.}
For each graph \(G_p=(\ccalV_p,\ccalE_p)\) let
\(n_p=\lvert \ccalV_p\rvert\), \(e_p=\lvert \ccalE_p\rvert\) and
\(d_p=\max_{v\in \ccalV_p}\deg(v)\) be the number of nodes, number of edges, and maximum degree of the graph $G_p$.
ORCA \citep{orca} counts all \(2\!-\!4\)-node graphlets in  
\[
T_{\text{ORCA}}(G_p)=O\!\bigl(e_p d_p + n_p d_p^{3}\bigr).
\]
Because every graph can be processed independently, we dispatch the
\(P\) graphs to \(M\) workers (\(M\!\le\!P\)).
Hence the \emph{wall-clock} preprocessing time is  
\[
T_{\text{stage\,1}}\;=\;
O\!\Bigl(
    \bigl\lceil\tfrac{P}{M}\bigr\rceil
    \max_{p}\!\bigl(e_p d_p + n_p d_p^{3}\bigr)
\Bigr).
\]
With one worker per graph (\(M=P\)) this shrinks to the single-graph
cost that dominates (\(\max_p\)).

\paragraph{Stage 2: training the Moment network.}
 Define:
\begin{itemize}
  \item \(L\): number of Monte-Carlo samples per epoch;
  \item \(N_e\): number of training epochs;
  \item \(C_{\text{INR}}\): cost of one forward/back-prop through the INR
        for a single edge probability;
  \item \(| \theta |\): total number of trainable parameters.
\end{itemize}
Each motif instance $F$ of size \(|\ccalV_F| \le 4\) invokes the INR at most
six times, a constant.  
One epoch therefore costs 
\[
T_{\text{epoch}}
    =O\!\bigl(L\,C_{\text{INR}} + | \theta |\bigr),\qquad
T_{\text{stage\,2}}
    =O\!\bigl(N_e \,(L\,C_{\text{INR}} + |\theta |)\bigr).
\]

\paragraph{Overall wall-clock complexity.}
\[
T_{\text{MomentNet}}
  =O\!\Bigl(
      \bigl\lceil\tfrac{P}{M}\bigr\rceil
      \max_{p}(e_p d_p + n_p d_p^{3})
      + N_e \,(L\,C_{\text{INR}} + |\theta |)
     \Bigr).
\]

\subsection*{Comparison with SIGL in Sparse vs.\ Dense Regimes}

SIGL~\cite{azizpour2025scalable} requires message-passing GNN training,
histogram building and INR fitting; with \(N_e\) epochs its wall-clock
cost is \(T_{\text{SIGL}} = O(PN_e n_{T}^{2})\),
where \(n_{T}=\max_p n_p\).

\begin{itemize}
  \item \textbf{Sparse regime}  
        (\(d_{\max}=O(1)\Rightarrow e_p=O(n_p)\)):
        \begin{itemize}
          \item \textsc{MomentNet}:  
                \(T = O\!\bigl(\lceil\frac{P}{M}\rceil n_{T}
                             + N_e\,(L\,C_{\text{INR}}+|\theta |)\bigr)\);
          \item SIGL: \(T = O(PN_e n_{T}^{2})\).
        \end{itemize}
        Here MomentNet grows \emph{linearly} in \(n_{T}\) (plus the
        network-training term), whereas SIGL is quadratic.  
        In practice we repeatedly observe MomentNet to be faster when
        graphs have \(e_p=O(n_p)\) even for very large \(n_p\).

  \item \textbf{Dense regime}  
        (Erdős–Rényi with \(p_{conn}\!=\!0.5\) implies
        \(d_{\max}\!\approx\!n_{T}/2\) and \(e_p=\Theta(n_{T}^{2})\)):
        \begin{itemize}
          \item \textsc{MomentNet}:  
                \(T = O\!\bigl(
                       \lceil\frac{P}{M}\rceil \, n_{T}^{4}
                       + N_e\,(L\,C_{\text{INR}}+|\theta |)
                     \bigr)\);
          \item SIGL: \(T = O(PN_e n_{T}^{2})\).
        \end{itemize}
        Asymptotically, SIGL’s \(n_{T}^{2}\) term is smaller than MomentNet’s \(n_{T}^{4}\). Yet empirical runs on dense ER graphs with \(p_{conn}\!=\!0.5\) still show MomentNet to be faster once (i)~Stage 1 is fully parallelised and (ii)~the constants behind GNN message passing and histogramming dominate SIGL’s quadratic term. Thus, the theoretical advantage of SIGL in dense graphs does not necessarily translate into shorter wall-clock times. Furthermore, MomentNet utilizes a two-stage process. The initial stage involves computing motif counts from the input graphs. Following this, the graphs are discarded. The second stage, which our experiments show to be the dominant phase of our method, then trains an Implicit Neural Representation (INR) using a vector of average moments derived from these counts. This design provides a significant reason for our method's improved speed, particularly in dense scenarios. By isolating the computationally expensive motif counting to a preliminary step, this cost is bypassed during the subsequent, dominant INR learning phase.
\end{itemize}

\noindent

With graph-level parallelism, \textsc{MomentNet} is
\emph{provably linear} in the number of edges for sparse networks and
remains competitive on dense networks because its constant factors are
smaller and its training cost is independent of the graph size.

%%%%%%%%%%%%%%%%%%%%%%%%%%%%%%%%%%%%%%%%%%%%%%%%%%%%%%%%%%%%%%%%%%%%%%%%

%%%%%%%%%%%% Proof of theorem 2 %%%%%%%%%%%%%%%%%%%%
% ------------------------------------------------------------------
\section{Proof of Proposition 1}\label{app:proof_prop1}
% ------------------------------------------------------------------
Let $W_1,W_2\!:\,[0,1]^2\to[0,1]$ be two graphons and fix 
$\alpha\in(0,1)$.  
Denote their convex combination by
\[
   W_\alpha \;=\;\alpha\,W_1 + (1-\alpha)\,W_2.
\]

%---------------- Edge moment (linear) -----------------------------
\paragraph{Edge density (a linear functional).}
For the single–edge motif $F_e$ on vertices $\ccalV_{F_e} = \{1,2\}$ and $\ccalE_{F_e} = \{(1,2)\}$, the induced
density is  
\[
   t^{\prime}(F_e,W)\;=\;
      \int_{[0,1]^2} W(\eta_1,\eta_2)\,d\eta_1\,d\eta_2 
      \;=\; \mathbb{E}[W(\eta_1,\eta_2)].
\]
Because the integrand is \emph{linear} in $W$, we immediately have
\[
   t^{\prime}(F_e,W_\alpha)
      \;=\;\alpha\,t^{\prime}(F_e,W_1)
         + (1-\alpha)\,t^{\prime}(F_e,W_2),
\]
so the edge density behaves affinely under convex combinations.

%---------------- V–shape moment (non-linear) -----------------------
\paragraph{The $V$–shape motif.}
Let $F$ be the \emph{$V$–shape} (three-vertex path) on vertex set
$\mathcal{V}_F=\{1,2,3\}$ and edge set  
$\mathcal{E}_F=\{(1,2),(1,3)\}$.  
Its induced density is
\begin{equation}
   t^{\prime}(F,W)
      \;=\;
      \int_{[0,1]^3}
         W(\eta_1,\eta_2)\,
         W(\eta_1,\eta_3)\,
         \bigl[1-W(\eta_2,\eta_3)\bigr]\,
      d\eta_1d\eta_2d\eta_3.
\label{eq:vshape}
\end{equation}

Plugging $W_\alpha$ into \eqref{eq:vshape}
\begin{align}
   t^{\prime}(F,W_\alpha)
   =&
   \mathbb{E}\!\Bigl[
      \bigl(\alpha W_1 + (1-\alpha)W_2\bigr)_{12}\,
      \bigl(\alpha W_1 + (1-\alpha)W_2\bigr)_{13}\,
      \bigl(1- \alpha W_{1}-(1-\alpha)W_{2}\bigr)_{23}
   \Bigr] \nonumber \\
   =&
      \alpha^3\,\mathbb{E}\!\bigl[(W_{1})_{12} (W_{1})_{13} (1-(W_{1})_{23}) \bigr]
      \;+\;(1-\alpha)^3\,\mathbb{E}\!\bigl[(W_{2})_{12} (W_{2})_{13} (1-(W_{2})_{23})\bigr] \nonumber \\
   &+\textbf{mixed\,terms},
\label{eq:vshape_expand}
\end{align}
where, to simplify notation, we used $(\cdot)_{ij}$ to denote that the graphon inside the parenthesis is evaluated on $(\eta_i, \eta_j)$ and ``mixed terms'' contain products in which at least one factor comes from $W_1$ and another from $W_2$. Because these mixed terms generally do not cancel, the right-hand side of~\eqref{eq:vshape_expand} \emph{does not} reduce to the affine combination
\begin{equation}
   \alpha\,t^{\prime}(F,W_1) + (1-\alpha)\,t^{\prime}(F,W_2),
\end{equation}
except in degenerate cases (e.g.\ $W_1=W_2$ or $\alpha\in\{0,1\}$).

%---------------- Explicit constant-graphon example -----------------
\paragraph{Concrete counter-example.}
Take constant graphons $W_1 (\eta_i, \eta_j) = p_1$ and $W_2 (\eta_i, \eta_j) = p_2$ with
$p_1$ and $p_2$ being constants satisfying $0<p_1\neq p_2<1$.  
Then $W_\alpha (\eta_i, \eta_j) =  p_\alpha = \alpha p_1+(1-\alpha)p_2$, and
\[
   t^{\prime}(F,W_i)=p_i^{\,2}(1-p_i),\qquad
   t^{\prime}(F,W_\alpha)=p_\alpha^{\,2}(1-p_\alpha),
\]
for $i \in \{1, 2\}$.
However,
\[
   p_\alpha^{\,2}(1-p_\alpha)
   \;\neq\;
   \alpha\,p_1^{\,2}(1-p_1) + (1-\alpha)\,p_2^{\,2}(1-p_2)
\]
whenever $p_1\neq p_2$ and $\alpha\in(0,1)$, confirming that the
$V$–shape moment is \emph{not} affine in $W$.

%---------------- Conclusion ----------------------------------------
\paragraph{Conclusion.}
Edge moments are linear in the graphon, but higher-order induced
moments involve \emph{non-linear} (polynomial) combinations of $W$.
Consequently, a convex combination of graphons preserves edge moments
but fails to preserve the remaining components of the motif-moment
vector.\qed

%%%%%%%%%%%%%%%%%%%% Methods Details %%%%%%%%%%%%%%%%%
\section{Methods Details}\label{app:method_details}

\subsection{Latent Variable Invariance of MomentNet}
The graphon model and our proposed model to learn it exhibit invariance to the specific ordering or labeling of latent variables. This means that the estimated graphon is unchanged under measure preserving transformations~\citep{borgs2008convergent}.
In other words, if the underlying structure of a graphon is rearranged or relabeled, MomentNet can still accurately capture the essential underlying connectivity patterns.
To illustrate this crucial property, we conduct an experiment using an SBM graphon, more precisely the one indexed by 12 in Table~\ref{table:groundtruth}.
For this experiment, we utilize the same dataset that was generated for the performance comparison of MomentNet discussed in Section~\ref{sec:exps}.
The learned graphons for three different realizations of this experiment are presented in Figure~\ref{fig:sbm_graphons}.
It is evident that all three estimated graphons closely resemble the ground truth graphon, which is depicted in Figure~\ref{fig:graphon_diagram}. Also, the three estimated graphons reflect the same underlying structure, and all of them share a similar GW loss, which is a loss function invariant to measure preserving transformations.
This essentially means that, no matter which of the three depicted graphons we sample graphs from, the underlying structure of all these graphs will be the same.
This outcome strongly verifies that MomentNet's primary mechanism involves matching the moments of the graph, without caring about the ordering of the latent variables.
Consequently, and in contrast to other methodologies, its estimated graphon accurately reflects the ground truth structure, allowing for differences only up to a permutation of the latent variable locations.

\begin{figure}[ht!] % 'h'ere, 't'op, 'b'ottom, 'p'age of floats
    \centering % Center the entire figure content

    % Top row with two images
    \begin{subfigure}[b]{0.48\textwidth} % 'b' for bottom alignment of subcaptions
        \centering
        \includegraphics[width=\linewidth]{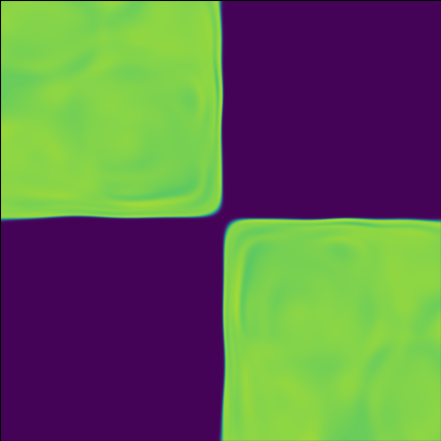}
        \caption{Estimated SBM graphon (Sample 1).} % Uncomment to add a caption
        \label{fig:sbm_graphon1}
    \end{subfigure}
    \hfill % Adds horizontal space between the two top images
    \begin{subfigure}[b]{0.48\textwidth}
        \centering
        \includegraphics[width=\linewidth]{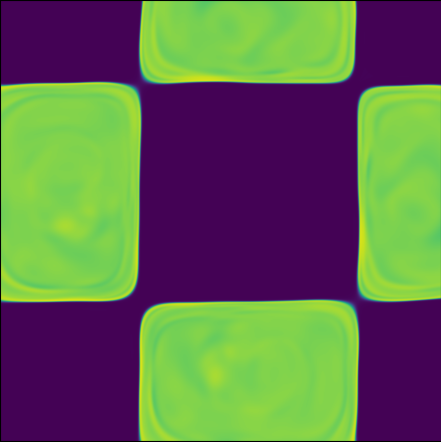}
        \caption{Estimated SBM graphon (Sample 2).} % Uncomment to add a caption
        \label{fig:sbm_graphon2}
    \end{subfigure}

    \vspace{0.5cm} % Optional: adds some vertical space between rows

    % Bottom row with one image in the middle
    \begin{subfigure}[b]{\textwidth} % Takes full width to allow centering of a smaller image
        \centering
        \includegraphics[width=0.5\linewidth]{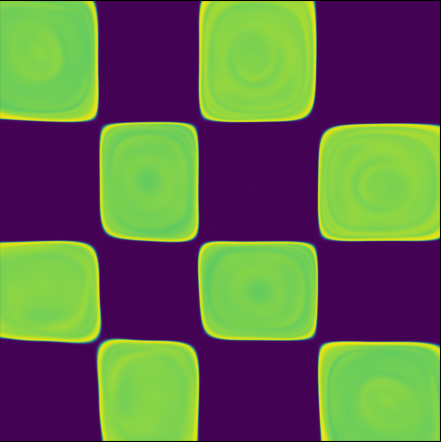} % Adjust width as needed
        \caption{Estimated SBM graphon (Sample 3).} % Uncomment to add a caption
        \label{fig:sbm_graphon3}
    \end{subfigure}

    \caption{Three samples of estimated graphons derived from a SBM.} % Uncomment to add an overall figure caption
    \label{fig:sbm_graphons}
\end{figure}

\subsection{MomentMixup Pseudocode}

\begin{algorithm}[H] % [htb] suggests float placement: here, top, bottom
\caption{MomentMixup Augmentation}
\label{alg:moment_mixup}
\begin{algorithmic}[1] % Enables line numbering
    \Require
        $\alpha_{\text{mix}}$: float, mixing coefficient ($0 \le \alpha_{\text{mix}} \le 1$).
        
        $\mathcal{G}_i, \mathcal{G}_j$: list of graphs, graph datasets for classes $i$ and $j$.
        
        $y_i, y_j$: integer, label for classes $i$ and $j$.
        
        $N_{\text{sample}}$: integer, number of graphs to sample from each class dataset to compute average moments.
        
        $N_{\text{nodes}}$: integer, number of nodes for each new graph.
        
        $N_{\text{graphs}}$: integer, number of augmented graphs to generate.
    \Ensure
        $\mathcal{G}_{\text{aug}}$: list of graphs and labels, newly generated augmented graphs.

    \State \Comment{Compute average moment vector for class $i$}
    \State $\mathcal{S}_i \gets \text{Randomly select } N_{\text{sample}} \text{ graphs from } \mathcal{G}_i$
    \State $\mathbf{m}_i \gets \frac{1}{N_{\text{sample}}} \sum_{G \in \mathcal{S}_i} \text{ComputeGraphMoments}(G)$
    
    \State \Comment{Compute average moment vector for class $j$}
    \State $\mathcal{S}_j \gets \text{Randomly select } N_{\text{sample}} \text{ graphs from } \mathcal{G}_j$
    \State $\mathbf{m}_j \gets \frac{1}{N_{\text{sample}}} \sum_{G \in \mathcal{S}_j} \text{ComputeGraphMoments}(G)$

    \State $\mathbf{m}_{\text{target}} \gets \alpha_{\text{mix}} \cdot \mathbf{m}_i + (1 - \alpha_{\text{mix}}) \cdot \mathbf{m}_j$ \Comment{Compute target mixed moments}

    \State $y_{\text{target}} \gets \alpha_{\text{mix}} \cdot y_i + (1 - \alpha_{\text{mix}}) \cdot y_j$
    \Comment{Compute the label for the new samples}

    \State $W_{\text{aug}} \gets \text{MomentNet}(\mathbf{m}_{\text{target}})$
    \Comment{Trains MomentNet for $\mathbf{m}_{\text{target}}$}% and returns its corresponding graphon $W_{\text{aug}}$}

    \State $\mathcal{G}_{\text{aug}} \gets \text{[]}$ \Comment{Initialize list for augmented samples}
    \For{$k \gets 1 \text{ to } N_{\text{graphs}}$}
        \State $G_{\text{new}} \gets \text{SampleGraph}(W_{\text{aug}}, N_{\text{nodes}})$
        \Comment{Sample new graph}
        \State Add ($G_{\text{new}}$, $y_{\text{target}}$) to $\mathcal{G}_{\text{aug}}$
    \EndFor
    \State \Return $\mathcal{G}_{\text{aug}}$
\end{algorithmic}
\end{algorithm}

%%%%%%%%%%%%%%%%%%%%%%%%%%%%%%%%%%%%%%%%%%%%%%%%%%%%%
%%%%%%%%%%%%%%%%%% Graphons %%%%%%%%%%%%%%%%%%%%

\section{List of Graphons}\label{app:graphon_table}
\begin{table*}[!htb]
\caption{Table of Graphons}
    \centering
    \begin{tabular}{lll}
        \toprule
         & $W(x,y)$ \\ 
         \midrule
1 &  $xy$  \\
2 &  $e^{(-(x^{0.7}+y^{0.7}))}$  \\
3 &  $\frac{1}{4}(x^2+y^2+\sqrt{x}+\sqrt{y})$ \\
4 &  $\frac{1}{2}(x+y)$ \\ 
5 & $(1+e^{(-2(x^2+y^2))})^{-1}$\\ 
6 & $(1+e^{(-\max\{x,y\}^2-\min\{x,y\}^4)})^{-1}$\\ 
7 & $e^{(-\max\{x,y\}^{0.75})}$\\
8 & $e^{(-\frac{1}{2}(\min\{x,y\}+\sqrt{x}+\sqrt{y}))}$\\
9 &  $\log(1+\max\{x,y\})$\\
10 &  $|x-y|$\\
11 & $1-|x-y|$\\
12 & $0.8\mathbf{I}_2\otimes\mathds{1}_{[0,\frac{1}{2}]^2}$\\
13 & $0.8(1-\mathbf{I}_2)\otimes\mathds{1}_{[0,\frac{1}{2}]^2}$\\
        \bottomrule
    \end{tabular}
    \label{table:groundtruth}
\end{table*}

The graphons are also visualized in Figure~\ref{fig:graphon_diagram}.

\begin{figure}[H]
    \centering
    \includegraphics[width=0.8\linewidth]{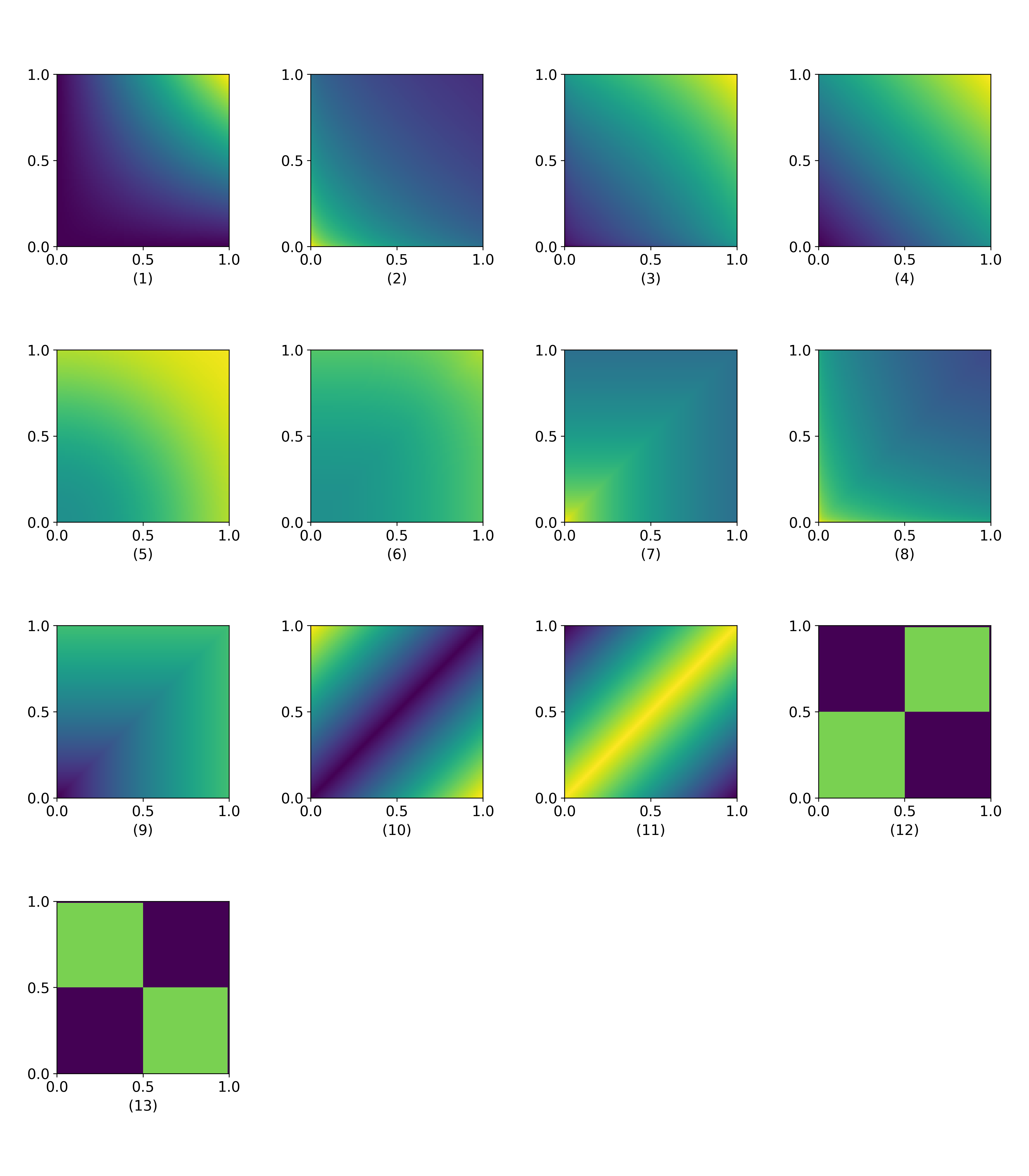}
    \caption{Representation of the graphons defined in Table~\ref{table:groundtruth}.}
    \label{fig:graphon_diagram}
\end{figure}

%%%%%%%%%%%%%%%%%%%%%%%%%%%%%%%%%%%%%%%%%%%%%

%%%%%%%%%%%%%%%%% Diag of Motfis %%%%%%%%%%%%%

\section{Selected Motifs}\label{app:motif_plot}

\begin{figure}[H]
\begin{tikzpicture}[
    every node/.style={circle, draw=black, fill=black, minimum size=2mm, inner sep=1pt},
    scale=0.7
]

% Define horizontal spacing
\def\dx{2.2}

% G0
\node (g0a) at (0,0) {};
\node (g0b) at (0,-1) {};
\draw (g0a) -- (g0b);

% G1
\node (g1a) at (\dx,0) {};
\node (g1b) at (\dx,-1) {};
\node (g1c) at (\dx,-2) {};
\draw (g1a) -- (g1b) -- (g1c);

% G2
\node (g2a) at (2*\dx,0) {};
\node (g2b) at (2*\dx+0.6,-1) {};
\node (g2c) at (2*\dx-0.6,-1) {};
\draw (g2a) -- (g2b) -- (g2c) -- (g2a);

% G3
\node (g3a) at (3*\dx,0) {};
\node (g3b) at (3*\dx,-1) {};
\node (g3c) at (3*\dx,-2) {};
\node (g3d) at (3*\dx,-3) {};
\draw (g3a) -- (g3b) -- (g3c) -- (g3d);

% G4
\node (g4c) at (4*\dx,-1) {}; % center
\node (g4a) at (4*\dx-0.7,0) {};
\node (g4b) at (4*\dx+0.7,0) {};
\node (g4d) at (4*\dx,-2) {};
\draw (g4c) -- (g4a) -- (g4c) -- (g4b);
\draw (g4c) -- (g4d);

% G5
\node (g5a) at (5*\dx-0.5,0) {};
\node (g5b) at (5*\dx+0.5,0) {};
\node (g5c) at (5*\dx+0.5,-1) {};
\node (g5d) at (5*\dx-0.5,-1) {};
\draw (g5a) -- (g5b) -- (g5c) -- (g5d) -- (g5a);

% G6
\node (g6a) at (6*\dx,0) {};
\node (g6b) at (6*\dx-0.6,-1) {};
\node (g6c) at (6*\dx+0.6,-1) {};
\node (g6d) at (6*\dx,1) {};
\draw (g6a) -- (g6b) -- (g6c) -- (g6a);
\draw (g6a) -- (g6d);

% G7
\node (g7a) at (7*\dx-0.5,0) {};
\node (g7b) at (7*\dx+0.5,0) {};
\node (g7c) at (7*\dx+0.5,-1) {};
\node (g7d) at (7*\dx-0.5,-1) {};
\draw (g7a) -- (g7b) -- (g7c) -- (g7d) -- (g7a);
\draw (g7a) -- (g7c);

% G8
\node (g8a) at (8*\dx-0.5,0) {};
\node (g8b) at (8*\dx+0.5,0) {};
\node (g8c) at (8*\dx+0.5,-1) {};
\node (g8d) at (8*\dx-0.5,-1) {};
\draw (g8a) -- (g8b) -- (g8c) -- (g8d) -- (g8a);
\draw (g8a) -- (g8c);
\draw (g8b) -- (g8d);
% Labels aligned below
\foreach \i in {0,...,8} {
  \node[text=black, fill=white, font=\scriptsize] at (\i*\dx,-3.8) {\( F_{\i} \)};
}

\end{tikzpicture}
\caption{Motifs up to four nodes.}
\label{fig:motifs_diag}
\end{figure}
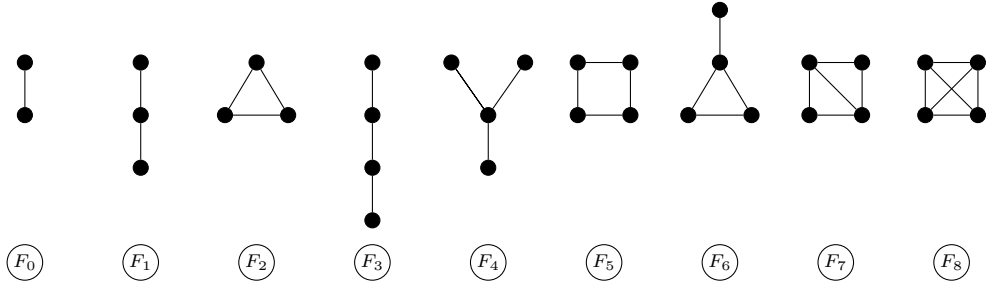

%%%%%%%%%%%%%%%%%%%%%%%%%%%%%%%%%%%%%%%%%%%%%%%

%%%%%%%%%%%%%%%%%%%%%% Centrality Measures Using Graphon %%%%%%%%%%%%%%%%%%
\section{Centrality Measures}\label{app:cent_measures}

In real-world graph statistical analysis, \textbf{centrality measures} are of significant interest to researchers. Building upon the work of Avella-Medina et al.~\cite{avella2020centrality}, who demonstrated the computability of these measures on graphons, we use several centrality metrics to further evaluate the quality of the estimated graphons. Specifically, we employ:

\begin{itemize}[leftmargin=*, wide, labelwidth=!, labelindent=0pt]
    \item \textbf{Degree Centrality}: This measure quantifies the number of direct connections a node possesses.
    \begin{itemize}
        \item \textit{High Value}: Indicates a node with many direct connections, often acting as a local hub with numerous immediate interactions. Such a node is highly active in its local neighborhood.
        \item \textit{Low Value}: Suggests a node with few direct connections, implying less immediate activity or influence within its local vicinity.
    \end{itemize}

    \item \textbf{Eigenvector Centrality}: This identifies influential nodes by considering that connections to other highly-connected (and thus influential) nodes contribute more significantly to a node's score. It measures how well-connected a node is to other well-connected nodes.
    \begin{itemize}
        \item \textit{High Value}: A node with high eigenvector centrality is connected to other nodes that are themselves influential. This node is likely a key player within an influential cluster or a leader among leaders.
        \item \textit{Low Value}: A node with low eigenvector centrality is typically connected to less influential nodes or has relatively few connections overall. Its influence is not strongly amplified by the influence of its neighbors.
    \end{itemize}

    \item \textbf{Katz Centrality}: This measure considers all paths in the graph, assigning exponentially more weight to shorter paths while still accounting for longer ones. It uses an attenuation factor $\alpha$, which determines the weight given to longer paths: smaller values of $\alpha$ emphasize shorter paths, while larger values give more importance to longer paths, up to a theoretical limit to ensure convergence.
    \begin{itemize}
        \item \textit{High Value}: Indicates a node that is reachable by many other nodes through numerous paths, with shorter paths contributing more. This node is generally well-connected throughout the network, both directly and indirectly, and can efficiently disseminate or receive information.
        \item \textit{Low Value}: Suggests a node that is not easily reachable by many other nodes or is primarily connected via very long paths. Its overall influence or accessibility within the network is limited.
    \end{itemize}

    \item \textbf{PageRank Centrality}: Originally developed for web pages, PageRank assesses a node's importance based on the number and quality of its incoming links. A link from an important node carries more weight than a link from a less important one. It uses a damping factor $\beta$, representing the probability that a random walker will follow a link to an adjacent node, while $(1-\beta)$ is the probability they will jump to a random node in the graph, ensuring that all nodes receive some rank and preventing rank-sinking in disconnected components.
    \begin{itemize}
        \item \textit{High Value}: A node with high PageRank centrality receives many ``votes'' (incoming connections) from other important nodes. This indicates that significant entities within the network consider this node to be important or authoritative.
        \item \textit{Low Value}: A node with low PageRank centrality receives few incoming connections or is primarily linked by less important nodes. It is not widely recognized as important by other influential nodes in the network.
    \end{itemize}
\end{itemize}

The mathematical formulations for these graphon-based centrality measures are adopted directly from Avella-Medina et al.~\cite{avella2020centrality}, corresponding to equations (7), (8), (9), and (10) in their paper, respectively. For a detailed analysis, we focus on graphons 1 and 2, as specified in Table~\ref{table:groundtruth}. We compute both analytical and sample-based centrality measures, establishing these as baselines for comparison with our results. The analytical computations directly apply the aforementioned formulas from Avella-Medina et al.~\cite{avella2020centrality}. For the sample-based approach, we generate discrete graph instances by drawing samples from the ground truth graphon and subsequently compute the centrality measures within this discrete domain. Further details regarding each graphon are presented in the subsequent subsections.

\subsection{Graphon 1: The ($xy$) Model}

The analytical centrality measures formulas for this graphon are as follows:
\begin{itemize}
    \item \textbf{Degree Centrality:}
    $$ C^{d}(x) = \frac{x}{2} $$

    \item \textbf{Eigenvector Centrality:}
    $$ C^{e}(x) = \sqrt{3}x $$

    \item \textbf{Katz Centrality:}
    $$ C^{k}_{\alpha}(x) =  (6 - 2\alpha) + 3\alpha x $$
    
    \item \textbf{PageRank Centrality:}
    $$ C^{\text{pr}}_{\beta}(x) = (1-\beta) + 2\beta x $$
  
\end{itemize}
These measures are for the given latent variable $x \in [0,1]$, after computing its centrality vector, we normalize it before comparison with discrete graph centralities~\cite{avella2020centrality}.
Since the ordering for these experiments is important, we create a new dataset of 20 graphs with 100 nodes each, preserving the latent variables for all the nodes. The experiment results are illustrated in Figure~\ref{fig:xy_cent}. Our results show that centrality measures from the MomentNet-predicted graphon (blue lines in the figure) are close to the analytical computations (ground truth, black dashed lines). Furthermore, these graphon-based centralities by MomentNet also provide a good approximation for centrality measures computed over discrete graph samples (red dots).

\begin{figure}[H]
    \centering
    \includegraphics[width=0.9\linewidth]{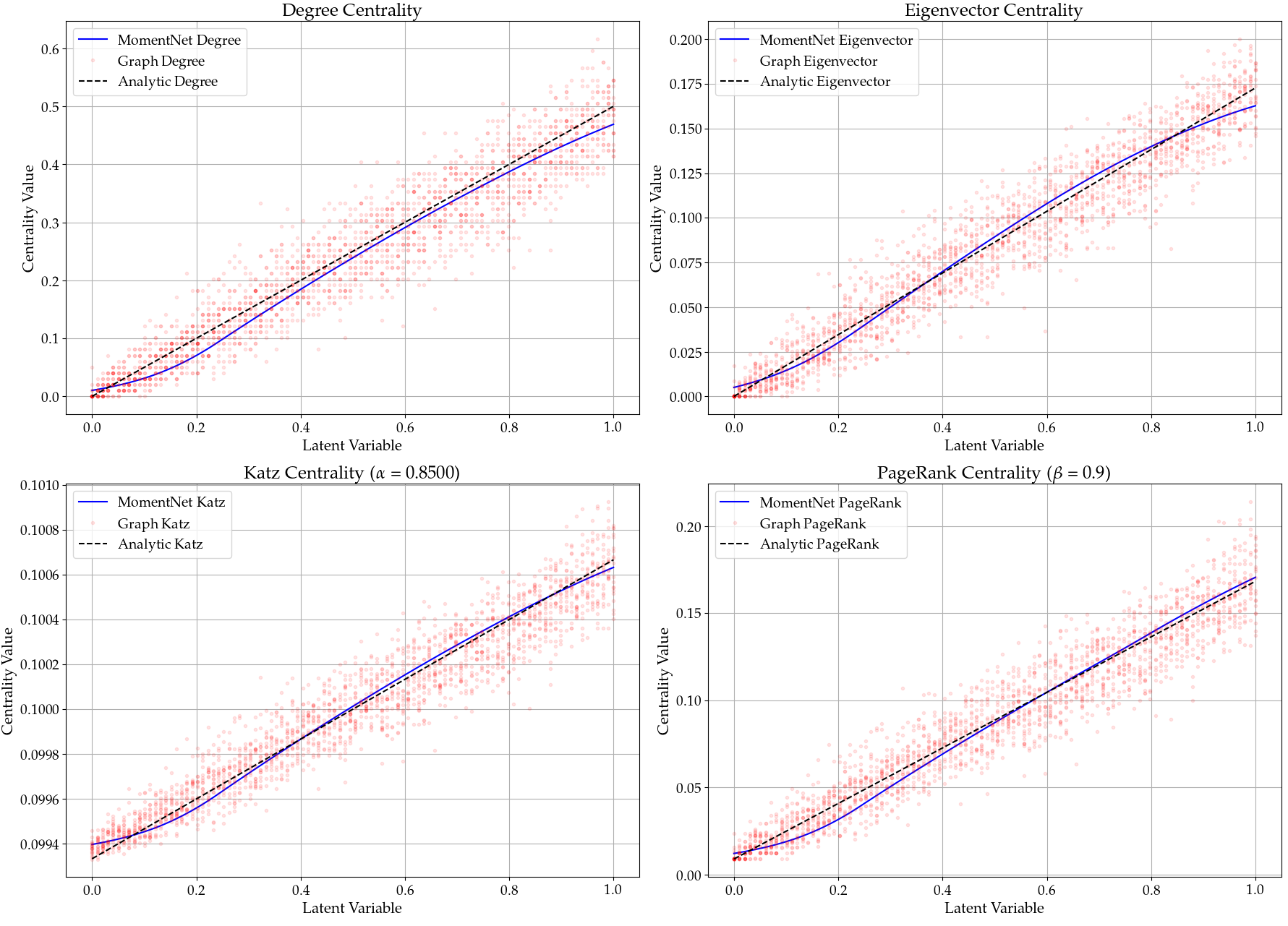}
    \caption{Centrality measures: MomentNet vs. analytic computation for the $xy$ graphon.}
    \label{fig:xy_cent}
\end{figure}

\subsection{Graphon 2: The ($e^{(-(x^{0.7}+y^{0.7}))}$) Model}
To test the generalizability and consistent performance of our method across varying complexities, we replicated the experiment on a more complex graphon. The analytical centrality measures formulas for this graphon are as follows:
\begin{itemize}
    \item \textbf{Degree Centrality:}
    $$ C^d(x) = 0.7492\,e^{-x^{0.7}} $$

    \item \textbf{Eigenvector Centrality:}
    $$ C^{e}(x) = \frac{e^{-x^{0.7}}}{\sqrt{0.473}} $$

    \item \textbf{Katz Centrality:}
    $$ C^{k}_{\alpha}(x) = 1 + \frac{0.7492\,\alpha\,e^{-x^{0.7}}}{1-0.473\,\alpha} $$

    \item \textbf{PageRank Centrality:}
    $$ C^{\text{pr}}_{\beta}(x,\beta) = (1-\beta) + \frac{\beta}{0.7492}\,e^{-x^{0.7}} $$
\end{itemize}

The experiment results are illustrated in Figure~\ref{fig:exp_cent}. Similar to the previous experiment, after computing the centrality measures on the graphon and analytically, we normalize them to compare them with the discrete graph measurement. As the plots show, similar to the previous graphon, our estimation is very close to the ground truth results obtained by analytical calculation.

\begin{figure}[ht!]
    \centering
    \includegraphics[width=0.9\linewidth]{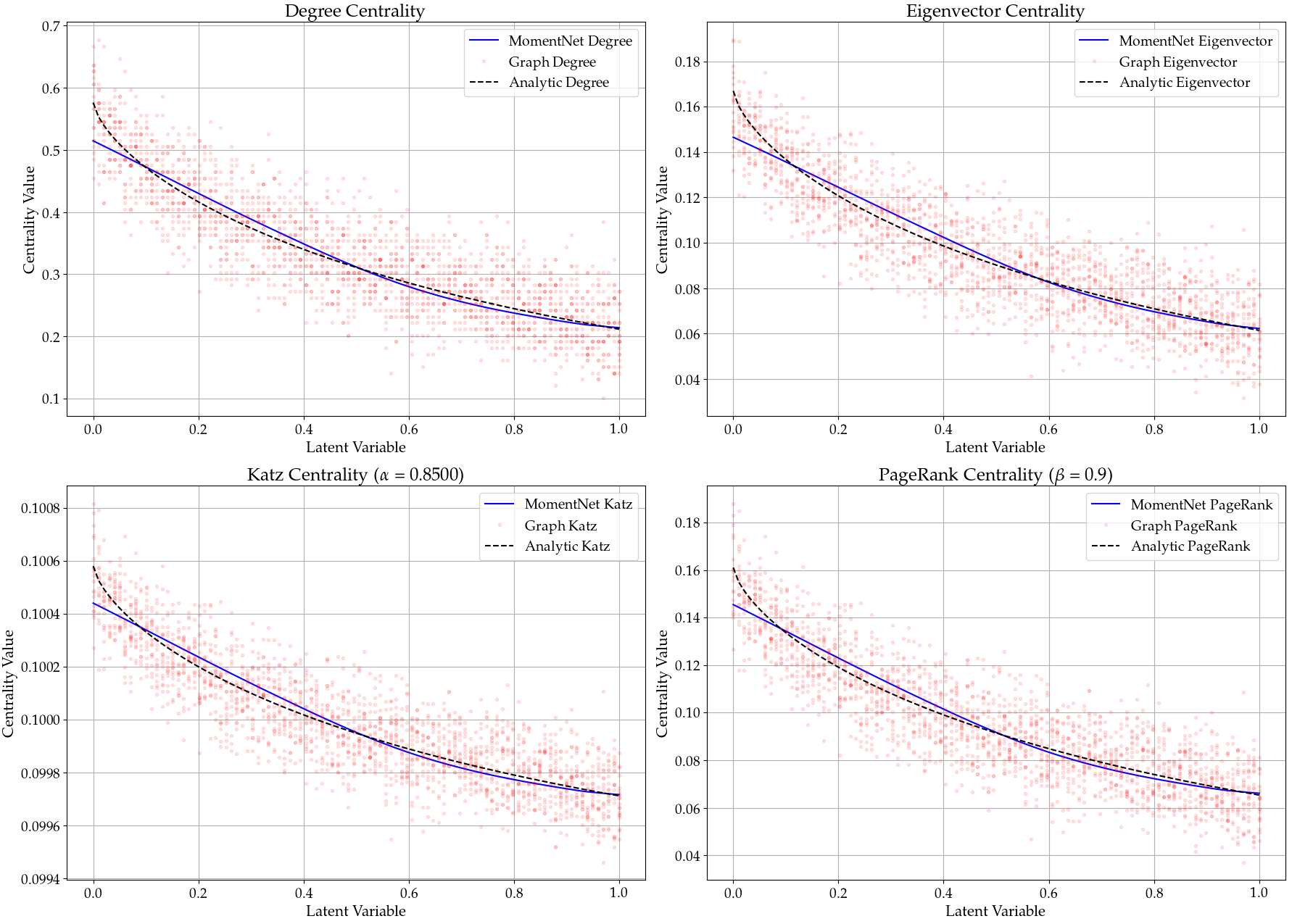}
    \caption{Centrality measures: MomentNet vs. analytic computation for the $e^{(-(x^{0.7}+y^{0.7}))}$ graphon.}
    \label{fig:exp_cent}
\end{figure}

%%%%%%%%%%%%%%%%%%%%%%%%%%%%%%%%%%%%%%%%%%%%%%%%%%

%%%%%%%%%%%%%%%%%%%%%% Extra Scalability Evaluation %%%%%%%%%%%%%%%%%%
\section{Extra Scalability Evaluations}\label{app:scale_exp}
We conducted an additional experiment to evaluate the scalability of SIGL and MomentNet. For this assessment, rather than focusing on SIGL's known weaknesses in latent variable estimation, we selected graphon number 5 from Table~\ref{table:groundtruth}, a model that both methods accurately estimate. We generate 10 graphs for each node size $n \in \{10, 20, \hdots, 810\}$.

Figure~\ref{fig:momentnet_scale_graphon5} illustrates the scalability of MomentNet and SIGL in terms of both performance, measured by GW loss, and average runtime, as a function of the number of nodes.
Subfigure (a) of Figure~\ref{fig:momentnet_scale_graphon5} reveals that MomentNet (blue line) maintains a consistently low GW loss across the tested range of node sizes, indicating stable performance. In contrast, SIGL's (red line) GW loss starts notably higher for smaller networks but decreases substantially as the number of nodes increases, eventually matching or even slightly outperforming MomentNet's loss for larger networks.

However, subfigure (b) of Figure~\ref{fig:momentnet_scale_graphon5} highlights a significant difference in computational efficiency: MomentNet's average runtime exhibits only a modest and gradual increase with the number of nodes. Conversely, SIGL's runtime escalates sharply, demonstrating significantly poorer scalability.

Consequently, while SIGL might offer a marginal advantage in GW Loss for very large graphs, MomentNet's vastly superior runtime scalability makes it a more practical and favorable approach, particularly for applications involving large-scale networks where computational resources and time are critical factors.

\begin{figure}[H] % Placement options: h-here, t-top, b-bottom, p-page
    \centering % Center the entire figure content

    % Second row: two plots side by side
    \begin{subfigure}[b]{0.48\textwidth} % Adjust width as needed, ensures they fit side-by-side
        \centering
        \includegraphics[width=\linewidth]{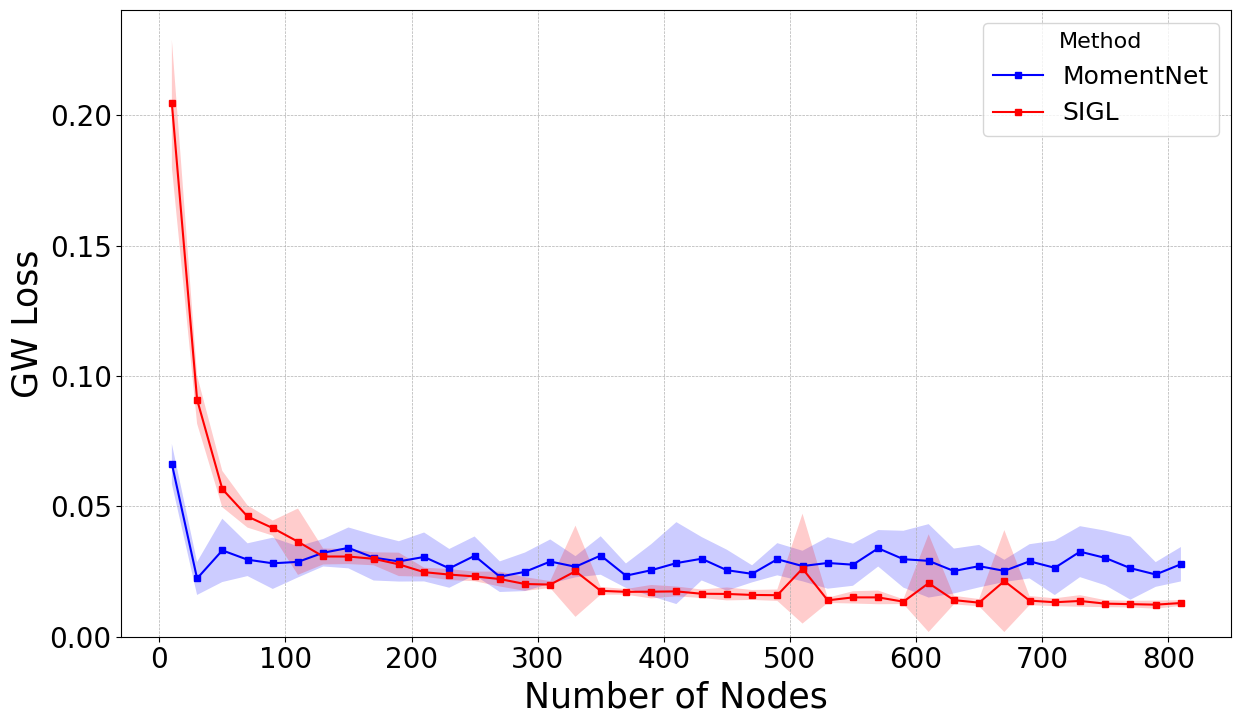} % Image takes full width of its subfigure
        \subcaption{Comparison of performance scalability of MomentNet with SIGL.}
        \label{fig:performance_scalability_graphon5} % Label for referencing
    \end{subfigure}
    \hfill % Automatically distributes horizontal space between the two subfigures
    \begin{subfigure}[b]{0.48\textwidth} % Adjust width as needed
        \centering
        \includegraphics[width=\linewidth]{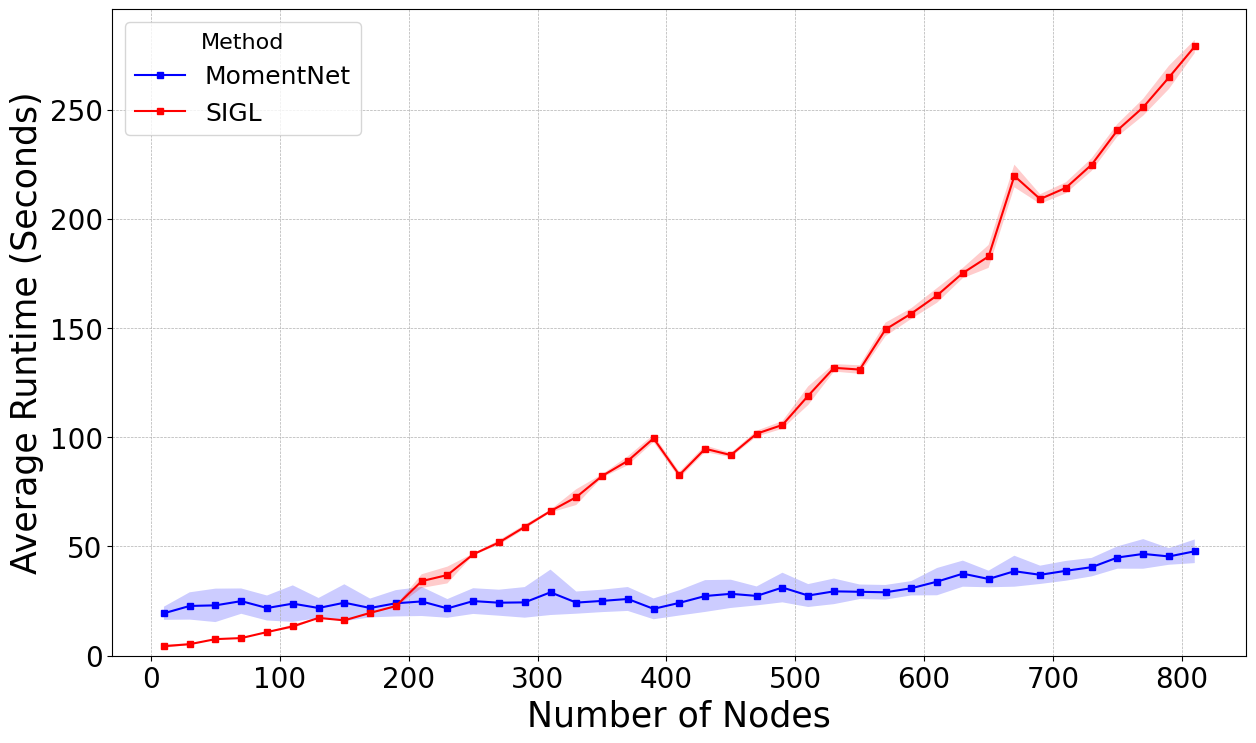} % Image takes full width of its subfigure
        \subcaption{Comparison of runtime scalability of MomentNet with SIGL.}
        \label{fig:runtime_scalability_graphon5} % Label for referencing
    \end{subfigure}

    % Optional: A main caption for the entire figure
    \caption{Scalability Comparison of MomentNet and SIGL}    \label{fig:momentnet_scale_graphon5}
\end{figure}

\mycomment{For our scalability analysis, we compare our method against \textbf{SIGL}. 
We also present a scenario where \textbf{SIGL} exhibits limitations in accurately estimating graphons, even when the observed graph is large. This challenge could arise because \textbf{SIGL}'s performance primarily relies on the accurate estimation of latent variables. In the standard graphon model, graph edges are generated from a graphon $W$ through Bernoulli trials with edge probabilities $W(\eta_i, \eta_j)$. If these probabilities are consistently near $0.5$ across the graphon domain, the variance of the edge generation process, given by $W(\eta_i, \eta_j)(1-W(\eta_i, \eta_j))$, approaches its maximum. Such high variance can obscure the underlying latent structure, thereby complicating the recovery of latent variables from the observed graph. A graphon exemplifying this challenge is $W(\eta_i, \eta_j) = 0.5 + 0.1 \cos(\pi \eta_i) \cos(\pi \eta_j)$. This particular graphon, which maintains edge probabilities close to $0.5$, is visualized in Figure~\ref{fig:MomentNet}. 
\sms{Do not explain all these stuff first! Explain first the experiment and the result and THEN explain the reason why you see the results you see. Right now you start from the answer even before you showed me the result}
For our experimental evaluations, we generate 10 graphs for each node size $n \in \{10, 20, \ldots, 510\}$ and estimate the associated graphon as described next, and we report the averaged results over 10 realizations. We consider two distinct scenarios:
\sms{What if you only present the `multiple graph' version of the experiment? the curves for single do not seem to add too much and the discussion would be simpler}
\begin{itemize}
    \item \textbf{Single-Graph Estimation:} In this scenario, a graphon is estimated using a single graph of size $n$.
    \item \textbf{Multiple-Graph Estimation:} Here, we generate 10 graphs of size $n$ each, and obtain the target motif counts by averaging the individual motif counts of each graph.
    %To ensure robust results and account for any stochasticity in the estimation algorithm, this process of estimating the graphon from the set of 10 graphs is repeated 10 times, and the performance metrics are then averaged across these repetitions.
    This approach allows the estimation method to leverage a more comprehensive set of samples from the underlying graphon in each estimation run compared to the single-graph setting.
\end{itemize}
The experiment results are provided in Figure~\ref{fig:performance_scalability}, and~\ref{fig:runtime_scalability}. In single-graph scenarios, \textbf{MomentNet}'s error is initially high for small graphs but decreases significantly as the number of nodes increases. Utilizing multiple graphs, however, enables \textbf{MomentNet} to achieve near-optimal performance, even on these small graphs. In both cases, this is due to a better estimation of the motif densities, which agrees with the intuition developed in the theory section that a larger sample size, both in terms of graph size and number of graphs, leads to a better performance of our method.
While \textbf{SIGL} demonstrates a similar dependency on node count in the single-graph setting, its performance does not substantially improve with the use of multiple graphs; it offers only slight gains, primarily for small graphs \sms{Any explanation why?}.
Consequently, \textbf{SIGL}'s overall performance is inferior to that of \textbf{MomentNet}.
%For \textbf{MomentNet}, the observed performance lower bound, which does not improve even with an increasing number of nodes or graphs, is attributable to the specific number of motifs employed in the experimental setup.
Regarding the runtime, Figure~\ref{fig:runtime_scalability} reveals that the average runtime of \textbf{MomentNet (Multiple)}, while exhibiting variance, scales more favorably with an increasing number of nodes compared to \text{SIGL (Multiple)}.
Specifically, \textbf{MomentNet (Multiple)} demonstrates a clear runtime advantage over \text{SIGL (Multiple)} when the number of nodes becomes large (exceeding approximately $300$ in the depicted data).
In contrast, for scenarios involving a single graph, \text{SIGL (Single)} consistently maintains a lower average runtime than \textbf{MomentNet (Single)} across all tested node counts.
The notable variance in runtime for the MomentNet approaches (\textbf{MomentNet (Multiple)} and \textbf{MomentNet (Single)}), represented by the shaded areas, is attributed to the early stopping criteria employed during the training process.
A more detailed analysis of computational complexity is provided in Appendix~\ref{app:complexity}.
\sms{Why are we putting all these things in bolded font? No need to do that}}

%%%%%%%%%%%%%%%%%%%%%%%%%%%%%%%%%%%%%%%%%%%%%%%%%%

%%%%%%%%%%%%%% MomentMixup Evaluation Details %%%%%%
\section{MomentMixup Evaluation Details}\label{app:mixup_details}

Our experimental evaluation is conducted on four diverse benchmark datasets widely used in graph classification research.
Table \ref{tab:dataset_descriptions} provides a detailed overview of these datasets, outlining their specific characteristics and the nature of their respective classification tasks.

\begin{table}[H]
\centering
\caption{Description of the benchmark datasets used for evaluation. Each dataset represents a different type of graph structure and classification task.}
\label{tab:dataset_descriptions}
\begin{tabular}{p{0.15\textwidth} p{0.5\textwidth} p{0.2\textwidth} p{0.1\textwidth}}
\toprule
\textbf{Dataset} & \textbf{Description} & \textbf{Classification Task} & \textbf{Citation} \\
\midrule
IMDB-B & Movie collaboration graphs; nodes represent actors/actresses, and an edge connects two actors/actresses if they appear in the same movie. & Binary genre classification. & ~\cite{datasets} \\
\midrule
IMDB-M & A multi-class version of IMDB-B, representing movie collaborations with similar graph construction. & Multi-class genre classification. & ~\cite{datasets} \\
\midrule
REDD-B & Social network graphs from Reddit; nodes represent users, and an edge indicates an interaction (e.g., one user commented on another's post). & Binary community (subreddit) classification. & ~\cite{datasets} \\
\midrule
AIDS & Bioinformatics graphs representing molecules; nodes are atoms, and edges are covalent bonds between them. & Binary classification based on anti-HIV activity (active vs. inactive). & ~\cite{AIDS} \\
\bottomrule
\end{tabular}
\end{table}

%%%%%%%%%%%%%%%%%%%%%%%%%%%%%%%%%%%%%%%%%%%

\end{document}